\newif\ifshortver
\shortverfalse
\newif\ifasymm
\asymmfalse

\ifshortver
\documentclass[conference]{IEEEtran}
\else
\documentclass[english,onecolumn]{IEEEtran}
\fi

\usepackage[utf8]{inputenc} 
\usepackage[T1]{fontenc}    
\usepackage{hyperref}       
\usepackage{url}            
\usepackage{booktabs}       
\usepackage{amsfonts}       
\usepackage{nicefrac}       
\usepackage{microtype}      
\usepackage{xcolor}         
\usepackage{amsmath}
\usepackage{amsthm}
\usepackage{graphicx}
\usepackage{subcaption}

\makeatletter
\def\useAmsThmOrIEEE{0}
\ifnum\useAmsThmOrIEEE=1
\usepackage{setspace}

\theoremstyle{plain}
\newtheorem{thm}{\protect\theoremname}
\theoremstyle{definition}
\newtheorem{defn}[thm]{\protect\definitionname}
\theoremstyle{plain}
\newtheorem{prop}[thm]{\protect\propositionname}
\theoremstyle{plain}
\newtheorem{lem}[thm]{\protect\lemmaname}
\theoremstyle{plain}
\newtheorem{cor}[thm]{\protect\corollaryname}
\theoremstyle{definition}
\newtheorem{example}[thm]{\protect\examplename}
\theoremstyle{definition}
\newtheorem{rem}[thm]{\protect\remarkname}
\else
\newtheorem{thm}{\protect\theoremname}
\newtheorem{defn}[thm]{\protect\definitionname}
\newtheorem{prop}[thm]{\protect\propositionname}
\newtheorem{lem}[thm]{\protect\lemmaname}

\fi

\usepackage{mathrsfs}
\usepackage{algorithm,algpseudocode}
\usepackage{algorithmicx}

\usepackage{colortbl}
\definecolor{lightgray}{rgb}{0.9,0.9,0.9}
\definecolor{lightred}{rgb}{1,0.8,0.8}
\definecolor{lightgreen}{rgb}{0.6,1,0.6}
\definecolor{lightyellow}{rgb}{1,1,0.5}
\definecolor{lightgrey}{rgb}{0.8,0.8,0.8}

\allowdisplaybreaks[1]

\makeatother

\usepackage{babel}
\providecommand{\corollaryname}{Corollary}
\providecommand{\definitionname}{Definition}
\providecommand{\propositionname}{Proposition}
\providecommand{\theoremname}{Theorem}
\providecommand{\lemmaname}{Lemma}
\providecommand{\remarkname}{Remark}

\title{Communication-Efficient and Privacy-Adaptable Mechanism for Federated Learning}

\author{%
  \IEEEauthorblockN{Chih Wei Ling$^{1,2,a}$, Chun Hei Michael Shiu$^{3,d}$, Youqi Wu$^{4,e}$, Jiande Sun$^{6,g}$, Cheuk Ting Li$^{5,f}$  \\ Linqi Song$^{1,2,b}$, Weitao Xu$^{1,2,c}$\\}
  \IEEEauthorblockA{
  \ifshortver
  $^1$Department of Computer Science,  
  City University of Hong Kong
  $^2$City University of Hong Kong Shenzhen Research Institute \\
  $^3$Department of Electrical and Computer Engineering, University of British Columbia \\
  Department of \{$^4$Computer Science, 
  $^5$Information Engineering\}, The Chinese University of Hong Kong \\
  \else
  $^1$Department of Computer Science,  
  City University of Hong Kong, \\
  $^2$City University of Hong Kong Shenzhen Research Institute \\
  $^3$Department of Electrical and Computer Engineering, University of British Columbia \\
  Department of 
  \{$^4$Computer Science, 
  $^5$Information Engineering\}, The Chinese University of Hong Kong \\
  \fi
  $^6$School of Information Science and Engineering, Shandong Normal University}
  
  Email: \{$^a$cwling6,$^b$linqi.song,$^c$weitaoxu\}@cityu.edu.hk, $^d$shiuchm@student.ubc.ca,
  \{$^e$youqiwu,$^f$ctli\}@link.cuhk.edu.hk, $^g$jiandesun@hotmail.com 
  \thanks{This project was supported in part by the National Natural Science Foundation of China (Grant No. 62371411), National Key R\&D Program of China (Grant No. 2023YFE0208800), the Research Grants Council of the Hong Kong Special Administrative Region, China (Project No. CityU 11202124, CityU 11201422, and GRF CityU 11217823), NSF of Guangdong Province (Project No. 2024A1515010192), the Innovation and Technology Commission of Hong Kong (Project No. MHP/072/23).}
}

\begin{document}

\maketitle

\begin{abstract}
Training machine learning models on decentralized private data via federated learning (FL) poses two key challenges: communication efficiency and privacy protection.
In this work,
we address these challenges within the trusted aggregator model by introducing a novel approach called the Communication-Efficient and Privacy-Adaptable Mechanism (CEPAM), achieving both objectives simultaneously.
In particular, CEPAM leverages the rejection-sampled universal quantizer (RSUQ), a construction of randomized vector quantizer whose resulting  distortion is equivalent to a prescribed noise, such as Gaussian or Laplace noise, enabling joint differential privacy and compression. 
Our CEPAM provides the additional benefit of privacy adaptability, allowing clients and the server to customize privacy protection based on required accuracy and protection.
We theoretically analyze the privacy guarantee of CEPAM and investigate the trade-offs among user privacy and accuracy of CEPAM through experimental evaluations.
Moreover, we assess CEPAM's utility performance using MNIST dataset,
demonstrating that CEPAM surpasses baseline models in terms of learning accuracy.
\end{abstract}

\section{Introduction and Motivation}

Federated learning (FL) \cite{McMahan2016FL} enables the training of machine learning models using vast amounts of data held privately by edge devices by exploiting their computational capabilities \cite{Chen2019DLEdgeCompute}, thanks to its decentralized design \cite{Jeffrey2012distributed}. 
However, FL presents challenges \cite{Tian2019FLChallenges}, such as ensuring efficient communication during training phases and preserving the privacy of the sensitive data at edge devices. 
These challenges are commonly addressed separately through methods that introduce some distortion to updated models, like compression methods \cite{Konecn2016NIPSFederatedLS,lin2020DGCreducing,Corentin2017DLadapCompr,alistarh2017QSGD,shlezinger2020uveqfed} and differential privacy (DP) mechanisms \cite{Dwork06DP,Dwork14Book}.

To address both challenges simultaneously, a common approach involves initially applying a DP mechanism (such as Laplace \cite{Dwork06DP} or Gaussian \cite{Dwork14Book}) followed by quantization \cite{girgis2021shuffled,agarwal2021skellam}.
However, this approach has two drawbacks.
Firstly, the overall error (comprising privacy-preserving and quantization errors) will no longer be exact, 
leading to a degradation in the accuracy of 
model updates and the final trained model.
Secondly, the quantization noise does not contribute to enhancing privacy against the honest but inquisitive clients within the FL framework. 
In essence, the DP-then-quantize approach is suboptimal. 

An alternative approach relies on randomized quantization, such as subtractive dithering 
\cite{roberts1962picture,schuchman1964dither, Limb1969visual, Jayant1972speech, sripad1977necessary}. 
Previous works \cite{Kim2021FederatedLW,shahmiri2024communication,hasircioglu2023communication,yan2023layered,hegazy2024compression} have employed this method to develop a joint privacy and quantization mechanism for FL. 
However, these approaches are limited to scalar quantization, which is less effective than vector quantization in terms of compression ratio. 
Moreover, Lang et al. \cite{lang2023joint} investigated the use of subtractive dithering to generate privacy-preserving noises such as multivariate $t$-distributed  and Laplace noises for local DP. However, their proposed mechanism did not encompass the Gaussian mechanism.\footnote{ The mechanism proposed in \cite{lang2023joint} adds a carefully designed, low-power privacy-preserving noise to the model update before the dithered quantization step. However, the authors do not specify how to design the requisite infinitesimal noise to generate the Gaussian noise for local DP, a task that is highly non-trivial. 
In contrast, our proposed CEPAM specifies how to generate Gaussian noise for central DP.}

Channel simulation \cite{bennett2002entanglement,harsha2010communication,sfrl_trans,havasi2019minimal,shah2022optimal,flamich2023greedy,liu2024universal,Li2024ChannelSim} is technique that can achieve the same goal. 
Recognizing that the Gaussian distribution can be represented as a mixture of uniform distributions \cite{WalkerUniformP1999,CHOY2003exppower}, Agustsson et al.  \cite{agustsson2020universally} combined subtractive dithering with layered construction \cite{wilson2000layered} to simulate 1D Gaussian channel.
This approach was further developed in \cite{hegazy2022randomized} to simulate any 1D unimodal additive noise distribution using randomized scaling and offsets. 
Recently, a new construction in \cite{Ling2024arxivRSUQ} named RSUQ, which is a randomized vector quantization, 
was introduced to exactly simulate any multivariate nonuniform continuous additive noise distribution with finite communication by extending 1D methods to multivariate distributions.
However, \cite{Ling2024arxivRSUQ} did not consider FL setup.

In this work, we propose a joint mechanism, CEPAM, aimed at enhancing privacy while compressing model updates in FL. 
Our proposed CEPAM 
leverages RSUQ \cite{Ling2024arxivRSUQ}, building upon its capability to convert  quantization distortion into an additive noise term with adjustable variance, independent of the quantized data.
Within CEPAM, both clients and the server can tailor the privacy protection mechanism, offering the choice between the exact Gaussian mechanism 
(CEPAM-Gaussian)
or Laplace 
mechanism (CEPAM-Laplace), 
based on the required accuracy level and privacy protection. 
Additionally, the inherent nature of RSUQ as a vector quantizer  gives CEPAM the advantage in reducing the compression ratio compared to its scalar counterparts \cite{hasircioglu2023communication, yan2023layered,agustsson2020universally,wilson2000layered,hegazy2022randomized}.
We theoretically analyze the privacy guarantee of CEPAM
and investigate the trade-offs among user privacy and accuracy of CEPAM through experimental evaluations.
Additionally, we analyze the privacy of CEPAM-Gaussian and CEPAM-Laplace using techniques from privacy amplification \cite{Balle2018PrivAmpl, Feldman2018PrivAmpl}.
We validate our theoretical results by comparing CEPAM-Gaussian against other baselines using the MNIST dataset with MLP and CNN architectures. 
The experimental results demonstrate that for MLP and CNN models, CEPAM-Gaussian achieves improvements of 0.6-2.0\% and 0.4-1.0\% in accuracy, respectively, compared to other baselines.

The rest of this paper is organized as follows. Section~\ref{sec:rel_works} summarizes recent related works and compares them to our approach. Section~\ref{sec:model} reviews necessary preliminaries, presents the system model, and identifies the requirements for joint quantization and privacy in FL. Section~\ref{sec:cepam} details the proposed joint quantization and privacy scheme, and Section~\ref{sec:perf_analysis} provides its theoretical performance analysis. Finally, experimental results are presented in Section~\ref{sec:eval}.

\section{Related Works} \label{sec:rel_works}

The channel simulation problem \cite{bennett2002entanglement, winter2002compression, cuff2013distributed} is closely related to Shannon's lossy compression problem.
In this setup \cite{bennett2002entanglement}, an encoder compresses an input source $X$ into a message $M$, which a decoder then uses to generate an output $M$.
The encoder and decoder may share common randomness.
The objective is to ensure that $Y$ follows a prescribed conditional distribution given $X$.
Channel simulation was initially studied in the asymptotic setting \cite{bennett2002entanglement, winter2002compression, cuff2013distributed}, where $X$ is a sequence, and has also been investigated in the one-shot setting \cite{harsha2010communication,braverman2014public,sfrl_trans,li2021unified}, where $X$ is a scalar.
RSUQ \cite{Ling2024arxivRSUQ} can be regarded as a one-shot channel simulation scheme for simulating an additive noise channel with an arbitrary continuous noise distribution.
For a comprehensive treatment of the subject, we direct the interested reader to the monograph in \cite{Li2024ChannelSim}.

Federated learning (FL) was initially introduced by McMahan et al. \cite{McMahan2016FL} as a collaborative learning approach that allows model training without the need to gather user-specific data.
Research in FL focuses mainly on improving communication efficiency \cite{McMahan2016FL,Konecn2016NIPSFederatedLS,alistarh2017QSGD,wangni2017gradSpar,stich2018sparsifiedSGD,stich2018local,shlezinger2020uveqfed,li2020FedProx,Li2020On}, ensuring data privacy \cite{Bonawitz2016SecureAggre,geyer2018DPFL}, or an approach that addresses both aspects \cite{hasircioglu2023communication,yan2023layered,liu2024universal}.
This work is focused on an approach that addresses the issues of communication efficiency and data privacy simultaneously.

The communication cost in federated learning can be mitigated through two primary methods: reducing the number of communication rounds or compressing model updates (e.g., parameter vectors or stochastic gradients) before transmission. 
While the original FedAvg algorithm \cite{McMahan2016FL} operates heuristically under heterogeneous data (i.e., non-iid data) and device inactivity (i.e., stragglers), Local SGD explicitly reduces communication frequency by performing multiple local update steps before synchronization. However, the convergence guarantees for Local SGD, which were first established by Stich \cite{stich2018local} for strongly-convex and smooth objectives, relied on the idealized assumptions of iid data and no stragglers. 
Li et al. \cite{Li2020On} later extended this analysis by providing theoretical guarantees for Local SGD in settings with heterogeneous data and inactive devices.
Subsequently, Shlezinger et al. \cite{shlezinger2020uveqfed} integrated compression into this framework with UVeQFed, which incorporates quantization for heterogeneous settings.

Another major method to reduce communication costs is the compression of stochastic gradients prior to transmission \cite{alistarh2017QSGD,stich2018sparsifiedSGD,shlezinger2020uveqfed}. This includes methods like quantization, which constrains the number of bits used per value. A key technique is unbiased quantization \cite{alistarh2017QSGD}, which ensures the compressed gradients remain unbiased estimators of their true values. 
In contrast, sparsification techniques \cite{wangni2017gradSpar, stich2018sparsifiedSGD} improve communication efficiency by reducing the number of non-zero entries in the stochastic gradients.
Building on these advances, we propose a novel Local SGD variant that unifies quantization with a privacy-preserving mechanism based on RSUQ, simultaneously addressing communication efficiency and data confidentiality within a single framework.

A recent trend involves developing approaches that jointly address communication efficiency and privacy \cite{Kim2021FederatedLW,shah2022optimal,hasircioglu2023communication,yan2023layered,lang2023joint,hegazy2024compression,liu2024universal,lang2025olala}. 
For central DP, 
\cite{hasircioglu2023communication, yan2023layered} used subtractive dithering to achieve this for scalar entries.
CEPAM generalizes the scalar counterparts in \cite{hasircioglu2023communication,yan2023layered}, as it uses RSUQ, which is a vector quantizer.
For local DP, \cite{lang2023joint} proposed injecting low-power noise before the dithered quantization step, proving guarantees for Laplace distribution and multidimensional $t$-distribution. 
However, the authors do not specify how to generate the requisite infinitesimal noise to produce Gaussian noise. 
In contrast, our CEPAM directly specifies the generation of this Gaussian noise for central DP.
In parallel, techniques from channel simulation have inspired other DP compressors \cite{shah2022optimal, liu2024universal}.

\section{System Model and Preliminaries} \label{sec:model}

In this section, we describe the setup of FL framework. 
We begin by reviewing the conventional FL framework with bit-constrained model updates in Section~\ref{sec:FL}, followed by an exploration of LRSUQ in Section~\ref{sec:RSUQ}, and a discussion of differential privacy 
and privacy amplification 
in Section~\ref{sec:DP}.
Next, we formulate the problem and identify the desired properties used in the FL framework in Section~\ref{sec:problem}.

\subsection*{Notations}

Write $H(X)$ for the entropy and in bits.
Logarithms are to the base $2$.  
For $\mathcal{A},\mathcal{B}\subseteq\mathbb{R}^{n}$,
$\beta\in\mathbb{R}$, $\mathbf{G}\in\mathbb{R}^{n\times n}$, $\mathbf{x}\in\mathbb{R}^{n}$
write $\beta\mathcal{A}:=\{\beta\mathbf{z}:\,\mathbf{z}\in\mathcal{A}\}$,
$\mathbf{G}\mathcal{A}:=\{\mathbf{G}\mathbf{z}:\,\mathbf{z}\in\mathcal{A}\}$,
$\mathcal{A}+\mathbf{x}:=\{\mathbf{z}+\mathbf{x}:\,\mathbf{z}\in\mathcal{A}\}$,
$\mathcal{A}+\mathcal{B}=\{\mathbf{y}+\mathbf{z}:\,\mathbf{y}\in\mathcal{A},\,\mathbf{z}\in\mathcal{B}\}$
for the Minkowski sum, $\mathcal{A}-\mathcal{B}=\{\mathbf{y}-\mathbf{z}:\,\mathbf{y}\in\mathcal{A},\,\mathbf{z}\in\mathcal{B}\}$,
and $\mu(\mathcal{A})$ for the Lebesgue measure of $\mathcal{A}$.
Let $B^{n}:=\{\mathbf{x}\in\mathbb{R}^{n}:\,\Vert\mathbf{x}\Vert\le1\}$ be the unit $n$-ball.
For a function $f:\mathbb{R}^n \rightarrow \mathbb{R}$, its superlevel set is defined as $L_{u}^{+}(f):=\{\mathbf{x} \in \mathbf{R}^n:f(\mathbf{x}) \geq u\}$.

\subsection{Federated Learning (FL)} \label{sec:FL}

In this work, we consider the FL framework 
(or \emph{federated optimization})
proposed in \cite{McMahan2016FL}.
More explicitly, $K$ clients (or devices), each attached with a local dataset $\mathcal{D}^{(k)}$ where $k \in \{1, 2, ..., K\}  =: \mathcal{K}$, cooperate together to train a shared global model $\mathbf{W}$ with $m$ parameters through a central server.
The goal is to minimize the  objective function $F:\mathbb{R}^m \to \mathbb{R}$:
\begin{equation} \label{eq:FLOpt}
    \min_{\mathbf{W} \in \mathbb{R}^m} \Big\{F(\mathbf{W}):= \sum_{k \in \mathcal{K}} p_{k}F_{k}(\mathbf{W})\Big\},
\end{equation}
where $p_k$ is the weight of client $k$ such that $p_k \ge 0$ and $\sum_{k \in \mathcal{K}} p_{k} = 1$.
Suppose that the $k$-th local dataset contains $n_{k}$ training data: $\mathcal{D}^{(k)} = \{\xi_{k,1}, \dots, \xi_{k,n_{k}}\}$. The local objective function $F_k:\mathbb{R}^m \to \mathbb{R}$ is defined by
\begin{equation}
    F_{k}(\mathbf{W}) \equiv F_k(\mathbf{W},\mathcal{D}^{(k)}) := \frac{1}{n_k} \sum_{j=1}^{n_{k}} \ell(\mathbf{W};\xi_{k,j}),
\end{equation}
where $\ell(\cdot;\cdot)$ is an application-specified loss function.

Let $T$ denote the total number of iterations in FL and let $\mathcal{T}_T := \{0, \tau, 2\tau, \ldots, T\}$ denote the set of integer multiples of some positive integer $\tau \in \mathbb{Z}^+$ where $T \equiv 0 \pmod{\tau}$, called the set of \emph{synchronization indices}.
We describe one FL round of the conventional FedAvg  \cite{McMahan2016FL} for solving optimization problem \eqref{eq:FLOpt} as follows. 
Let $\mathbf{W}_t$ denote global parameter vector available at the  server at the time instance $t \in \mathcal{T}_T$.
At the beginning of each FL round, the server broadcasts $\mathbf{W}_t$ to all the clients.
Then, each client $k$ sets $\mathbf{W}_{t}^{k}=\mathbf{W}_{t}$ and computes the $\tau \; (\ge 1)$ local parameter vectors by SGD:\footnote{When $T$ is fixed, the larger $\tau$ is, the fewer the communication rounds there are.}
\begin{equation} \label{eq:local_sgd}
    \mathbf{W}_{t+t'}^{k} \leftarrow \mathbf{W}_{t+t'-1}^{k} - \eta_{t+t'-1} \nabla F_{k}^{j_{t+t'-1}^{k}}(\mathbf{W}_{t+t'-1}^{k}),\;\; t'=1, \ldots, \tau,
\end{equation}
where $\eta_{t+t'-1}$ is the learning rate, $\nabla F_{k}^{j}(\mathbf{W}):=\nabla F_{k}(\mathbf{W};\xi_{k,j})$ is the gradient computed at a single sample of index $j$, and $j_{t}^{k}$ is the sample index chosen uniformly from the local data $\mathcal{D}^{(k)}$ of client $k$ at time $t$.\footnote{In this work, our focus is on analyzing the computation of a single stochastic gradient at each client during every time instance.
The FL convergence rates can potentially be enhanced by incorporating mini-batching technique \cite{stich2018local}, leaving the detailed analysis for future work.}
Finally, by assuming that all clients participate in each FL round for simplicity, the server aggregates the $K$ local model updates $\{\mathbf{X}_{t+\tau}^{k}=\mathbf{W}_{t+\tau}^{k}-\mathbf{W}_t\}_{k\in \mathcal{K}}$ and computes the new global parameter vector:
\begin{equation} \label{eq:global_without_quan}
    \mathbf{W}_{t+\tau} \leftarrow \mathbf{W}_{t} +  \sum_{k \in \mathcal{K}} p_{k}\mathbf{X}_{t+\tau}^{k}= \sum_{k \in \mathcal{K}} p_{k}\mathbf{W}_{t+\tau}^{k}.
\end{equation}

The dataset $\mathcal{D}^{(k)}$ inherently induces a distribution. 
By an abuse of notation, we also denote this induced distribution as $\mathcal{D}^{(k)}$.
Suppose the data in client $k$ is iid sampled from the induced distribution $\mathcal{D}^{(k)}$.
Consequently, the overall distribution becomes a mixture of all local distributions: $\mathcal{D} = \sum_{k \in \mathcal{K}} p_k \mathcal{D}^{(k)}$.
Previous works typically assume that the data is iid generated by or partitioned among the $K$ clients, i.e., for all $k \in \mathcal{K}$, $\mathcal{D}^{(k)} = \mathcal{D}$.
However, we consider a scenario where the data is non-iid (or heterogeneous), implying that $F_k$ could potentially be an arbitrarily poor approximation to $F$.

In \cite{shlezinger2020uveqfed}, Shlezinger et al. introduced a framework for \emph{Quantized Federated Learning}. 
Given the limited bandwidth of uplink channel in FL framework, each $k$-th client is required to communicate a quantized version \cite{gray1998quantization} of the model update $\mathbf{X}_{t+\tau}^k$ using a finite number of bits to the server. 
Specifically, the $k$-th model update $\mathbf{X}_{t+\tau}^k$ is encoded into a binary codeword $u_{t+\tau}^{k}$ of length $R_k$ bits through an encoding function $\mathrm{Enc}_{t+\tau}^{k}:\mathbb{R}^m\to\{0,1,\ldots,2^{R_k}-1\}:=\mathcal{U}_k$. 
Since the outputs $\mathrm{Enc}_{t+\tau}^k(\mathbf{X}_{t+\tau}^k) = u_{t+\tau}^k \in \mathcal{U}_k$ of an encoder may have unequal probabilities, entropy coding \cite{huffman1952method,Golomb1966enc,elias1975universal} can be utilized to further reduce redundancy when transmitting through the binary lossless channel.
In the FL literature, the uplink channel is typically modeled as a bit-constrained link \cite{shlezinger2020uveqfed}, and the transmission is assumed to be lossless.
Upon receiving the set of codewords $\{u_{t+\tau}^{k}\}_{k \in \mathcal{K}}$ from clients, the server uses the joint decoding function $\mathrm{Dec}_{t+\tau}:\mathcal{U}_1 \times \ldots \times \mathcal{U}_K \to \mathbf{R}^m$ to reconstruct  $\hat{\mathbf{X}}_{t+\tau} \in \mathbb{R}^m$, an estimate of the weighted average $\sum_{k \in \mathcal{K}}p_k\mathbf{X}_{t+\tau}^k$.

\subsection{Rejection-Sampled Universal Quantizer (RSUQ)} \label{sec:RSUQ}

We will now review  RSUQ \cite{Ling2024arxivRSUQ}, which is constructed based on the subtractive dithered  quantizer (SDQ) \cite{roberts1962picture,ziv1985universal,zamir1992universal}. 

Given a non-singular generator matrix
$\mathbf{G}\in\mathbb{R}^{n\times n}$,  a \emph{lattice} is the set $\mathbf{G}\mathbb{Z}^{n}=\{\mathbf{G}\mathbf{j}:\,\mathbf{j}\in\mathbb{Z}^{n}\}$. 
A bounded set $\mathcal{P}\subseteq\mathbb{R}^{n}$ is called a \emph{basic
cell} of the lattice $\mathbf{G}\mathbb{Z}^{n}$ if $(\mathcal{P}+\mathbf{G}\mathbf{j})_{\mathbf{j}\in\mathbb{Z}^{n}}$
forms a partition of $\mathbb{R}^{n}$ \cite{conway2013sphere, zamir2014}.
Specifically, the Voronoi
cell $\mathcal{V}:=\{\mathbf{x}\in\mathbb{R}^{n}:\,\arg\min_{\mathbf{j}\in\mathbb{Z}^{n}}\Vert\mathbf{x}-\mathbf{G}\mathbf{j}\Vert=\mathbf{0}\}$ is a basic cell.
Given a basic cell $\mathcal{P}$, we can define a \emph{lattice quantizer} $Q_{\mathcal{P}}:\mathbb{R}^{n}\to\mathbf{G}\mathbb{Z}^{n}$ such that $Q_{\mathcal{P}}(\mathbf{x})=\mathbf{y}$
where $\mathbf{y}\in\mathbf{G}\mathbb{Z}^{n}$ is the unique lattice
point that satisfies $\mathbf{x}\in-\mathcal{P}+\mathbf{y}$. 
The resulting quantization error $\mathbf{z}:=Q_{\mathcal{P}}(\mathbf{x})-\mathbf{x}$ depends deterministically on the input $\mathbf{x}$ and is approximately uniformly distributed over
 the basic cell of the lattice quantizer under some regularity assumptions. 
Therefore, it is often combined with probabilistic methods such as random dithering to construct SDQ  \cite{ziv1985universal,zamir1992universal}.

\begin{defn}
Given a basic cell $\mathcal{P}$ and a random dither $\mathbf{V}\sim\mathrm{Unif}(\mathcal{P})$, a \emph{subtractive dithered
quantizer} (SDQ) $Q_{\mathcal{P}}^{SDQ}:\mathbb{R}^{n}\times\mathcal{P}\to\mathbb{R}^{n}$ for an input $\mathbf{x} \in \mathbb{R}^n$ is given by  $Q_{\mathcal{P}}^{SDQ}(\mathbf{x},\mathbf{v})=Q_{\mathcal{P}}(\mathbf{x}-\mathbf{v})+\mathbf{v}$, where $Q_{\mathcal{P}}$ is the lattice quantizer. 
\end{defn}

A well-known property of SDQ is that the resulting quantization error is uniformly distributed over the basic cell of the quantizer and is statistically independent of the input signal \cite{schuchman1964dither,zamir1992universal,gray1993dithered,kirac1996results,zamir1996lattice}.

The quantization error incurred from the aforementioned quantization schemes approximately follows a uniform distribution over a basic cell of a lattice. 
It may be desirable to have the quantization error follow a uniform distribution over an arbitrary set, rather than being distributed uniformly over a basic cell. 
RSUQ is a randomized quantizer where the quantization error is uniformly distributed over a set $\mathcal{A}$, a subset of a basic cell. This quantization scheme is based on applying rejection sampling on top of SDQ.
Intuitively, we keep generating new dither signals until the quantization error falls in $\mathcal{A}$.\footnote{It is easy to see that the acceptance probability is $\mu(\mathcal{A})/\mu(\mathcal{P})$.}

\begin{defn} \cite[Definition 4]{Ling2024arxivRSUQ}
\label{def:rej_samp_quant} Given a basic cell $\mathcal{P}$ of the
lattice $\mathbf{G}\mathbb{Z}^{n}$, a subset $\mathcal{A}\subseteq\mathcal{P}$, and a sequence $S=(\mathbf{V}_{i})_{i\in\mathbb{N}^{+}}$, $\mathbf{V}_{1},\mathbf{V}_{2},\ldots\stackrel{iid}{\sim}\mathrm{Unif}(\mathcal{P})$
are i.i.d. dithers, 
the \emph{rejection-sampled universal quantizer} (RSUQ) $Q_{\mathcal{A},\mathcal{P}}:\mathbb{R}^{n}\times \prod_{i \in \mathbb{N}^{+}}\mathcal{P}_i\to\mathbb{R}^{n}$ for $\mathcal{A}$
against $\mathcal{P}$ is given by 
\begin{equation}
Q_{\mathcal{A},\mathcal{P}}(\mathbf{x},(\mathbf{v}_{i})_{i}):=Q_{\mathcal{P}}(\mathbf{x}-\mathbf{v}_{h})+\mathbf{v}_{h},\label{eq:QAP_def}
\end{equation}
where
\begin{equation}
h:=\min\big\{ i:\,Q_{\mathcal{P}}(\mathbf{x}-\mathbf{v}_{i})+\mathbf{v}_{i}-\mathbf{x}\in\mathcal{A}\big\},\label{eq:kstar}
\end{equation}
and $Q_{\mathcal{P}}$ is the lattice quantizer for basic cell
$\mathcal{P}$.
\end{defn}

Note that SDQ is a special case of RSUQ where $\mathcal{A}=\mathcal{P}$. 
It is easy to check that the quantization error is uniform over $\mathcal{A}$ by using the standard rejection sampling argument and the ``crypto'' lemma \cite[Lemma 4.1.1 and Theorem 4.1.1]{zamir2014}. 

\begin{prop} \label{prop:error_dist_unif} \cite[Proposition 5]{Ling2024arxivRSUQ}
For any $\mathbf{x} \in \mathbb{R}^n$, the quantization error $\mathbf{Z} := Q_{\mathcal{A},\mathcal{P}}(\mathbf{x},S) - \mathbf{x}$ of RSUQ $Q_{\mathcal{A},\mathcal{P}}$, where $S \sim P_S$, 
follows the uniform distribution over the set $\mathcal{A}$, i.e., $\mathbf{Z} \sim \mathrm{Unif}(\mathcal{A})$.
\end{prop}

By using layered construction as in \cite{hegazy2022randomized,wilson2000layered}, RSUQ can be generalized to simulate an additive noise channel with noise following a continuous  distribution.
Consider a pdf $f:\mathbb{R}^{n}\to[0,\infty)$ and write its superlevel set as
\[
L_{u}^{+}(f)=\{\mathbf{z}\in\mathbb{R}^{n}:\,f(\mathbf{z})\ge u\}.
\]
Let $f_{U}(u):=\mu(L_{u}^{+}(f))$ for $u>0$, which is also a pdf. 
If we generate $U\sim f_{U}$, and then $\mathbf{Z}|\{U=u\}\sim\mathrm{Unif}(L_{u}^{+}(f))$,
then we have $\mathbf{Z}\sim f$ by the fundamental theorem of simulation \cite{robert2004monte}.

\begin{defn} \cite[Definition 12]{Ling2024arxivRSUQ}
\label{def:rej_samp_quant_layer} Given a basic cell $\mathcal{P}$
of the lattice $\mathbf{G}\mathbb{Z}^{n}$, a probability density
function  $f:\mathbb{R}^{n}\to[0,\infty)$ where $L_{u}^{+}(f)$ is
always bounded for $u>0$, and $\beta:(0,\infty)\to[0,\infty)$ satisfying
$L_{u}^{+}(f)\subseteq\beta(u)\mathcal{P}$ for $u>0$, and a random pair $S=(U,(\mathbf{V}_{i})_{i\in\mathbb{N}^{+}})$ where the latent variable 
$U\sim f_{U}$ with $f_{U}(u):=\mu(L_{u}^{+}(f))$, and $\mathbf{V}_{1},\mathbf{V}_{2},\ldots\stackrel{iid}{\sim}\mathrm{Unif}(\mathcal{P})$
is a sequence of i.i.d. dither signals, the \emph{layered
rejection-sampled universal quantizer} (LRSUQ) $Q_{f,\mathcal{P}}:\mathbb{R}^{n}\times \mathbb{R} \times \prod_{i \in \mathbb{N}^{+}}\mathcal{P}_i\to\mathbb{R}^{n}$ for $f$ against $\mathcal{P}$
is given by 
\begin{equation}
Q_{f,\mathcal{P}}(\mathbf{x},u,(\mathbf{v}_{i})_{i}):=\beta(u)\cdot\big(Q_{\mathcal{P}}(\mathbf{x}/\beta(u)-\mathbf{v}_{h})+\mathbf{v}_{h}\big),
\end{equation}
where
\begin{equation}
h:=\min\big\{ i:\,\beta(u)\cdot\big(Q_{\mathcal{P}}(\mathbf{x}/\beta(u)-\mathbf{v}_{i})+\mathbf{v}_{i}\big)-\mathbf{x}\in L_{u}^{+}(f)\big\},
\end{equation}
and $Q_{\mathcal{P}}$ is the lattice quantizer for basic cell
$\mathcal{P}$.
\end{defn}

It can be shown that LRSUQ indeed gives the desired error distribution by Proposition~\ref{prop:error_dist_unif}.

\begin{prop} \label{prop:LRSUQ_error} 
\cite[Proposition 13]{Ling2024arxivRSUQ}
Consider LRSUQ $Q_{f,\mathcal{P}}$. 
For any random input $\mathbf{X}$, the quantization error
defined by $\mathbf{Z} := Q_{f,\mathcal{P}}(\mathbf{X},U,(\mathbf{V}_{i})_{i}) - \mathbf{X}$, 
follows the pdf $f$, independent of $\mathbf{X}$.
\end{prop}

\subsection{Differential Privacy (DP)} \label{sec:DP}

Since clients do not directly transmit their data to the  server, the original FL framework \cite{McMahan2016FL} provides a certain level of privacy.
Nonetheless, a significant amount of information can still be inferred from the shared data, e.g., model parameters induced by gradient descent, by potential eavesdroppers within the FL network \cite{Zhu2019DeepLeak,zhao2020idlg,huang2021GradInvAtt}. 
Consequently, a privacy mechanism, such as DP, is essential to protect the shared information.

In this section, we delve into a notion of (central) DP  \cite{Dwork06DP,Dwork14Book}.\footnote{Herein, DP refers to \emph{central} DP.} 
In DP, clients place trust in the server (or data curator) responsible for collecting and holding their individual data in a database $X \in \mathcal{X}$, where $\mathcal{X}$ denotes the collection of databases. 
The server then introduces privacy-preserving noise to the original datasets or query results through a randomized mechanism $\mathcal{F}$, producing an output $Y = \mathcal{F}(X) \in \mathcal{Y}$, where $\mathcal{Y}$ denotes the set of possible outputs, before sharing them with untrusted data analysts.
While this model requires a higher level of trust compared to the local model, it enables the design of significantly more accurate algorithms. 
Two databases $X$ and $X'$ are considered \emph{adjacent} if they differ in only one entry. 
 More generally, we can define a symmetric adjacent relation $\mathcal{R} \subseteq \mathcal{X}^2$ and say that $X$ and $X'$ are \emph{adjacent} databases if $(X, X') \in \mathcal{R}$.
Here, we review the definition of $(\epsilon, \delta)$-differentially private, initially introduced by Dwork et al. \cite{Dwork06DP}.

\begin{defn}[$(\epsilon, \delta)$-differential privacy \cite{Dwork06DP}] \label{def:edDP} 
    A randomized mechanism $\mathcal{F}:\mathcal{X} \rightarrow \mathcal{Y}$ with the associated conditional distribution $P_{Y|X}$ of $Y=\mathcal{F}(X)$ is $(\epsilon, \delta)$-\emph{differentially private} ($(\epsilon,\delta)$-DP) if for all $\mathcal{S} \subseteq \mathcal{Y}$ and for all $(X, X') \in \mathcal{R}$,
    \[
    \mathbb{P}(\mathcal{F}(X) \in \mathcal{S}) \leq e^{\epsilon}\mathbb{P}(\mathcal{F}(X') \in \mathcal{S})+\delta.
    \] 
When $\delta = 0$, we say that $\mathcal{F}$ is $\epsilon$-\emph{differentially private} ($\epsilon$-DP).
\end{defn}

As the applications of differential privacy continue to expand,  the issue of managing the privacy budget is gaining increasing attention, 
particularly focusing on \emph{privacy composition} and \emph{privacy amplification}. 
\emph{Privacy composition} \cite{Dwork14Book,Abadi2016DL,Mironov2017RDP} asserts that the privacy budgets of composition blocks accumulate, while 
\emph{privacy amplification} \cite{Balle2018PrivAmpl,Feldman2018PrivAmpl} provides tools to analyze and bound the privacy budget of a combination of a selected base privacy mechanism to be less than the privacy of its constituent parts.

In this work, we use the principle of privacy amplification by subsampling \cite{Balle2018PrivAmpl}, whereby the privacy guarantees of a differentially private mechanism are amplified by applying it to a random subsample of the dataset.
The problem of privacy amplification can be stated as follows. 
Let $\mathcal{F}:\mathcal{X} \rightarrow \mathcal{Y}$ be a privacy mechanism with privacy profile $\delta_{\mathcal{F}}$ with respect to the adjacent relation defined on $\mathcal{X}$, and let $s:\mathcal{W} \rightarrow \mathcal{X}$ be a subsampling mechanism.
Consider the subsampled mechanism $\mathcal{F}^{s}:\mathcal{W} \rightarrow \mathcal{Y}$ given by $\mathcal{F}^{s}(X) :=\mathcal{F}(s(X))$. 
The goal is to relate the privacy profile of $\mathcal{F}$ and $\mathcal{F}^s$.

For our purposes, it is  instrumental to express differential privacy in terms of $e^{\epsilon}$-\emph{divergence} \cite{Sason2016fdivInq}. 
Let $\mathfrak{M}(\mathcal{Y})$ denote the set of probability measures on the output space $Y$. 
The $e^{\epsilon}$-\emph{divergence} between two probability measures $\mu, \mu' \in \mathfrak{M}(\mathcal{Y})$ is defined as:  
\begin{equation}
   D_{e^{\epsilon}}(\mu||\mu'):=\sup_{\mathcal{S}:\mathrm{\;measurable\;}\mathcal{S}\subseteq \mathcal{Y}}(\mu(\mathcal{S})-e^{\epsilon} \mu'(\mathcal{S})).
\end{equation}
Note that $\mathcal{F}$ is $(\epsilon,\delta)$-differential privacy if and only if $D_{e^{\epsilon}}(\mathcal{F}(X)||\mathcal{F}(X'))\le \delta$ for every $(X,X') \in \mathcal{R}$. 

To study the relavant properties of $\mathcal{F}$ from a privacy amplification point of view, we review the notions of \emph{privacy profile} and \emph{group privacy profile} \cite{Balle2018PrivAmpl} as follows. The \emph{privacy profile} $\delta_{\mathcal{F}}$ of a mechanism $\mathcal{F}$ is a function associating to each privacy parameter $e^{\epsilon}$ a supremium on the $e^{\epsilon}$-divergence between the outputs of running the mechanism on two adjacent databases, i.e., 
\begin{equation}
    \delta_{\mathcal{F}}(\epsilon):=\sup_{(X,X') \in \mathcal{R}}D_{e^{\epsilon}}(\mathcal{F}(X)||\mathcal{F}(X')).
\end{equation} 
The \emph{group privacy profile} $\delta_{\mathcal{F},j} (\epsilon)$ ($j \ge 1$) is defined as:
\begin{equation}
    \delta_{\mathcal{F},j} (\epsilon):=\sup_{d(X,X')\le j}D_{e^{\epsilon}}(\mathcal{F}(X)||\mathcal{F}(X')),
\end{equation}
where $d(X,X'):=\min\{j: \exists X_1, \ldots,X_{j-1} \mathrm{\;with\;} (X,X_1), \ldots, (X_{j-1},X') \in \mathcal{R}\}$.
Note that $\delta_{\mathcal{F}}=\delta_{\mathcal{F},1}$.

We briefly review subsampling with replacement \cite{Balle2018PrivAmpl}. 
Let $\mathcal{U}$ be a set. 
We write $2^{\mathcal{U}}$ and $\mathbb{N}^{\mathcal{U}}$ for the collections of all sets and multisets, respectively, over $\mathcal{U}$.
For an integer $n \geq 0$, we also write $2_{n}^{\mathcal{U}}$ and $\mathbb{N}_{n}^{\mathcal{U}}$ for the collections of all sets and multisets containing exactly $n$ elements, where the elements are counted with multiplicity for multisets. 
Given a multiset $X \in \mathbb{N}^{\mathcal{U}}$ we write $X_a$ for the number of occurrences of $a \in \mathcal{U}$ in $X$. 
The support of a multiset $X$ is the defined as the set $\mathrm{supp}(X):=\{a \in\mathcal{U}:X_a > 0\}$. 
Given multisets $X, X' \in \mathbb{N}^{\mathcal{U}}$ we write $X' \subseteq X$ to denote that $X'_a \leq X_a$ for all $a \in \mathcal{U}$.
For order-independent datasets represented as multisets, the \emph{substitute-one} relation, denoted as $X \simeq_s X'$, holds whenever $\Vert X-X' \Vert_1=2$ and $|X|=|X'|$, i.e.,  $X'$
is obtained by replacing an element in $X$ with a different element from $\mathcal{U}$.
The subsampling with replacement mechanism $s_\tau:2_n^{\mathcal{U}} \rightarrow \mathfrak{M}(\mathbb{N}_\tau^{\mathcal{U}})$ takes a set of $x$ of size $n$ and outputs a sample from the multinomial distribution $\omega = s(X)$ over all multisets $Y$ of size $\tau \leq n$ and $\mathrm{supp}(Y) \subseteq X$. 
We further assume that the base mechanism $\mathcal{F}:\mathbb{N}_\tau^{\mathcal{U}} \rightarrow \mathcal{Y}$ is defined on the multisets and has privacy profile $\delta_{\mathcal{F}}$ with respect to $\simeq_s$. 
In \cite{Balle2018PrivAmpl}, Balle et al. demonstrated a bound on the privacy profile of the subsampled mechanism with replacement $\mathcal{F}^{s_m}:2_n^{\mathcal{U}} \rightarrow \mathcal{Y}$ with respect to $\simeq_s$ as follows.

\begin{thm} \cite[Theorem 10]{Balle2018PrivAmpl}
\label{thm:PrivAmpl}
Let $s_\tau:2_n^{\mathcal{U}}\rightarrow \mathfrak{M}(\mathbb{N}_\tau^{\mathcal{U}})$ be the subsampling with replacement mechanism and let $\mathcal{F}:\mathbb{N}_\tau^{\mathcal{U}} \rightarrow \mathcal{Y}$ be a base privacy mechanism. Then,  
for $\tilde{\epsilon} \ge 0$ and $\epsilon = \log \left(1 + p(e^{\tilde{\epsilon}} - 1)\right)$ where $p:=1- \left(1-\frac{1}{n}\right)^{\tau}$, we have 
\begin{equation} \label{eq:privAmpli_S}
    \delta_{\mathcal{F}(s_{\tau}(X))}(\epsilon) \le \sum_{j=1}^{\tau} \binom{\tau}{j}\left(\frac{1}{n}\right)^j\left(1-\frac{1}{n}\right)^{\tau-j}\delta_{\mathcal{F}(X),j}(\tilde{\epsilon}),
\end{equation}
where $\delta_{\mathcal{F}(\cdot)}(\cdot)$ is the privacy profile of the privacy mechanism $\mathcal{F}(\cdot)$ and $\delta_{\mathcal{F}(\cdot),\cdot}(\cdot)$ is the group privacy profile of $\mathcal{F}(\cdot)$.
\end{thm}

The above result guarantees that applying a differentially private mechanism to a random subsample of a dataset offers stronger privacy assurances than applying it to the entire dataset. 
This is because information about an individual cannot be leaked if that individual's data is not included in the subsample.

Recall that for a function $q:\mathcal{X} \to \mathbb{R}^n$ with global $\ell_2$-sensitivity $\Delta_{2}:=\sup_{(X,X') \in \mathcal{R}}\Vert q(X) - q(X')\Vert_2$, the Gaussian mechanism $\mathcal{F}(X)=q(X)+\mathcal{N}(\mathbf{0},\tilde{\sigma}^2\mathbf{I}_n)$ satisfies $(\epsilon,\delta)$-DP if $\tilde{\sigma}^2 \ge 2 \Delta_{2}^2 \log(1.25/\delta)/\epsilon^2$ and $\epsilon \in (0,1)$ \cite[Appendix A]{Dwork14Book}. In \cite{balle2018ImprovGauss}, 
Balle and Wang gave a new analysis of the Gaussian mechanism that is valid for all $\epsilon > 0$.

\begin{thm} \cite[Theorem 8]{balle2018ImprovGauss}
\label{thm:analyGaussian} 
Let $q:\mathcal{X} \rightarrow \mathbb{R}^n$ be a function with global $\ell_2$-sensitivity $\Delta_{2}:=\sup_{(X,X') \in \mathcal{R}}\Vert q(X) - q(X')\Vert_2$. 
Then, for any $\epsilon \ge 0$ and $\delta \in [0,1]$, the Gaussian mechanism $\mathcal{F}(X):=q(X)+\mathbf{Z}$ with $\mathbf{Z} \sim \mathcal{N}(\mathbf{0}, \tilde{\sigma}^2\mathbf{I}_n)$, where $\mathbf{I}_n$ is the $n \times n$ identity matrix, is $(\epsilon, \delta)$-DP if and only if
\begin{equation} \label{eq:GaussianPrivP}
\Phi\left(\frac{\Delta_{2}}{2\tilde{\sigma}} - \frac{\epsilon \tilde{\sigma}}{\Delta_{2}}\right) -  e^{\epsilon} \Phi\left(-\frac{\Delta_{2}}{2 \tilde{\sigma}} - \frac{\epsilon \tilde{\sigma}}{\Delta_{2}}\right) \le \delta,
\end{equation}
where $\Phi$ denotes the Gaussian cdf.
\end{thm}

The above result can be interpreted as providing
an expression for the privacy profile of the Gaussian mechanism in terms of the cdf of a standard Gaussian distribution.

We prove a simple consequence of  Theorem~\ref{thm:PrivAmpl} and Theorem~\ref{thm:analyGaussian} as follows.

\begin{lem} \label{lem:subsampledGauss}
Let $q:\mathcal{X} \rightarrow \mathbb{R}^n$ be a function with global $\ell_2$-sensitivity $\Delta_{2}:=\sup_{(X,X') \in \mathcal{R}}\Vert q(X) - q(X')\Vert_2$, and let $s_\tau:2_n^{\mathcal{U}}\rightarrow \mathfrak{M}(\mathbb{N}_\tau^{\mathcal{U}})$ be the subsampling with replacement mechanism. Then, for any $\epsilon \geq 0$ and $\delta \in [0, 1]$, the subsampled Gaussian mechanism $\mathcal{F}^{s_{\tau}}(X) = q \circ s(X) + N(\mathbf{0}, \tilde{\sigma}^2\mathbf{I}_n)$ is $(\epsilon, \delta)$-DP if and only if, for any  $\tilde{\epsilon} > 0$, $\epsilon = \log \left(1 + p(e^{\tilde{\epsilon}} - 1)\right)$ where $p:=1- \left(1-\frac{1}{n}\right)^{\tau}$  and 
\begin{equation} \label{eq:PrivAmplGSR}
\sum_{j=1}^{\tau} \binom{\tau}{j}\left(\frac{1}{n}\right)^j\left(1-\frac{1}{n}\right)^{\tau-j}\frac{(e^{\tilde{\epsilon}}-1)}{e^{\tilde{\epsilon}/j}-1}\left(\Phi\left(\frac{\Delta_{2}}{2\tilde{\sigma}} - \frac{\tilde{\epsilon} \tilde{\sigma}}{j\Delta_{2}}\right) -  e^{\tilde{\epsilon}/j} \Phi\left(-\frac{\Delta_{2}}{2 \tilde{\sigma}} - \frac{\tilde{\epsilon} \tilde{\sigma}}{j\Delta_{2}}\right)\right)
 \le \delta.
\end{equation}
\end{lem}
\begin{proof}
A standard group privacy analysis \cite{Vadhan2017} gives the privacy profile 
\[\delta_{\mathcal{F}(X),j}(\tilde{\epsilon})\le \frac{(e^{\tilde{\epsilon}}-1)}{e^{\tilde{\epsilon}/j}-1}\cdot \delta_{\mathcal{F}(X)}(\tilde{\epsilon}/j).\]
Thus, the RHS of \eqref{eq:privAmpli_S} in Theorem~\ref{thm:PrivAmpl} is upper-bounded by
\begin{align*}
    &\sum_{j=1}^{\tau} \binom{\tau}{j}\left(\frac{1}{n}\right)^j\left(1-\frac{1}{n}\right)^{\tau-j}\frac{(e^{\tilde{\epsilon}}-1)}{e^{\tilde{\epsilon}/j}-1} \cdot \delta_{\mathcal{F}(X)}(\tilde{\epsilon}/j)\\
    &\stackrel{(a)}{=} \sum_{j=1}^{\tau} \binom{\tau}{j}\left(\frac{1}{n}\right)^j\left(1-\frac{1}{n}\right)^{\tau-j}\frac{(e^{\tilde{\epsilon}}-1)}{e^{\tilde{\epsilon}/j}-1}\left(\Phi\left(\frac{\Delta_{2}}{2\tilde{\sigma}} - \frac{\tilde{\epsilon} \tilde{\sigma}}{j\Delta_{2}}\right) -  e^{\tilde{\epsilon}/j} \Phi\left(-\frac{\Delta_{2}}{2 \tilde{\sigma}} - \frac{\tilde{\epsilon} \tilde{\sigma}}{j\Delta_{2}}\right)\right),
\end{align*}
where (a) holds since the base mechanism $\mathcal{F}(\cdot)$ is Gaussian, which has the privacy profile given as the LHS of \eqref{eq:GaussianPrivP}. 
This completes the proof.
\end{proof}
  
Lemma~\ref{lem:subsampledGauss} implies that applying a subsampled  Gaussian mechanism to a random subsample (obtained by sampling with replacement) of a dataset in a single communication round of the FL framework provides stronger privacy guarantees than applying the Gaussian mechanism to the entire dataset. 
Since each $k$-th client performs $\tau$ iterations of local SGD (i.e., Equation~\eqref{eq:local_sgd}) on indices sampled uniformly from the dataset $\mathcal{D}^{(k)}$ (i.e., $p_1=p_2=\ldots=p_{|\mathcal{D}^{(k)}|}=1/|\mathcal{D}^{(k)}|=1/n$, 
with $n=|\mathcal{D}^{(k)}|$ in~\eqref{eq:privAmpli_S} or~\eqref{eq:PrivAmplGSR}, where $p_j$ denotes the probability that the data point $\xi_{k,j}$ is chosen at time instance $t$), this procedure is equivalent to subsampling with replacement, 
which outputs a multiset sample of size $\tau$ following a multinomial distribution $\mathrm{Mult}(\tau;p_1,\ldots,p_{|\mathcal{D}^{(k)}|})$.\footnote{See \cite[Lemma 3]{hasircioglu2023communication} for an analogous result for the subsampled Gaussian mechanism under Poisson subsampling.}
The random multiset sample is then fed into the function $q(\cdot)=(\mathbf{W}_{t+\tau}^{k}-\mathbf{W}_t)(\cdot)$. 
Refer to Theorem~\ref{thm:Gaussian_Priv}.

Another important privacy mechanism is the Laplace mechanism. The privacy profile of the Laplace mechanism is given as follows:
\begin{thm}\cite[Theorem 3]{Balle2018PrivAmpl}
Let $q:\mathcal{X} \rightarrow \mathbb{R}$ be a function with the global $\ell_1$-sensitivity $\Delta_{1}:=\sup_{(X,X') \in \mathcal{R}}|q(X) - q(X')|$ and $\mathcal{F}(X):=q(X)+Z$ with $Z \sim \mathrm{Lap}(0,b)$ be the Laplace mechanism.  
Then, for any $\epsilon > 0$, the privacy profile of $\mathcal{F}$ is given by $\delta_{\mathcal{F}}(\epsilon)=\left[1-e^{\frac{\epsilon-\Delta_{1}/b}{2}}\right]_{+}$, where $[\cdot]_{+}:=\max\{\cdot,0\}$.
\end{thm}

Note that $\delta_{\mathcal{F}}(\epsilon)=0$ for any $\epsilon \ge \Delta_{1}/b$, which implies the well-known fact that the Laplace mechanism with $b \ge \Delta_{1}/\epsilon$ is $\epsilon$-DP.

\subsection{Problem Setting} \label{sec:problem}

Our goal is to construct a mechanism that simultaneously addresses privacy requirements (protecting against honest but inquisitive clients within the FL network or inquisitive data analysts upon public release of the final trained model, see Section~\ref{sec:threatM}) and compression demands (for lossless uplink channels with limited bandwidth) in a single local model update at each client within the FL framework.
Since the distribution of the model parameters and/or the induced gradient is often unknown to the clients, our focus lies in creating a universal mechanism applicable to any random source. Furthermore, we are exploring privacy mechanisms that introduce noise customized to various noise distributions. 
While Laplace mechanism provides pure DP protection, 
Gaussian mechanism only provides approximate DP protection with a small failure probability \cite{Dwork06DP,Dwork14Book}. 
However, it is well-known that Gaussian mechanism support tractability of the privacy budget in mean estimation \cite{Dong2022GDP,Mironov2017RDP}, an important subroutine in FL.
Such schemes can be formulated as mappings from the local update $\mathbf{X}_{t}^{k} \in \mathbb{R}^m$ at client $k$ at the time instance $t$ to the estimated update $\hat{\mathbf{X}}_{t}^{k} \in \mathbb{R}^m$ at the server, aimed to achieve the following desired properties:
\begin{enumerate}
    \item Privacy requirement : The perturbed query function generated by the privacy mechanism must adhere to $(\epsilon,\delta)$-DP (or $\epsilon$-DP). 
    For instance, ensuring that the mapping of the average of local model updates across clients $\frac{1}{K}\sum_{k \in \mathcal{K}}\mathbf{X}_{t}^k$ to the average of estimated updates $\frac{1}{K}\sum_{k \in \mathcal{K}}\hat{\mathbf{X}}_{t}^k$ at the server satisfies $(\epsilon,\delta)$-DP. 
    \item  Compression/communication-efficiency : The estimation $\hat{\mathbf{X}}_{t+\tau}^{k}$ from client $k$ to the server should be represented by finite bits per sample.
    \item Universal source : the scheme should operate reliably irrespective of the distribution of $\mathbf{X}_t^k$ and without prior knowledge of it. 
    \item Adaptable noise : The noise $Z$ in the privacy mechanism is customizable according to the required accuracy level and privacy protection.
\end{enumerate}

\subsubsection{\textbf{Threat model}} \label{sec:threatM}
We adopt trusted aggregator model, i.e., the server is trusted.
Additionally, we assume that there are separate sources of shared randomness between each client and the server to perform quantization. 
Clients engaged in the FL framework are assumed to be honest yet inquisitive, meaning they comply with the protocol but may attempt to deduce sensitive client information from the average updates received by the server. 
Our goals are:
\begin{itemize}
    \item Protecting the privacy of each client's local dataset from other clients, as the updated model between rounds may inadvertently disclose sensitive information.
    \item Preventing privacy leaks from the final trained model upon completion of training, as it too may inadvertently reveal sensitive information. 
    This ensures the relevance of our solution in scenarios where clients are trusted, and the final trained model may be publicly released to third parties.
\end{itemize}

\section{CEPAM for FL} \label{sec:cepam}

In this section, we introduce CEPAM, which utilizes the inherent randomness in probabilistic quantized FL to provide privacy enhancement. 
(see Algorithm~\ref{alg:UnivPriv}).
Our approach modifies the schemes in  \cite{shlezinger2020uveqfed} and \cite{hasircioglu2023communication} by replacing universal quantization  \cite{ziv1985universal,zamir1992universal} with RSUQ  as outlined in Section~\ref{sec:RSUQ}. 
Through the utilization of RSUQ and the layered construction \cite{hegazy2022randomized,wilson2000layered}, leveraging on the unique quantization error property of RSUQ (refer to Proposition~\ref{prop:LRSUQ_error}), we construct a joint privacy and compression mechanism to address both FL challenges of communication overload and privacy requirement adaptable to any system adjustment simultaneously.   
This is achieved by applying norm clipping and then LRSUQ to the model updates $\mathbf{X}_{t+\tau}^k$ at the end of each FL round, separately for each client.
These quantized updates are then transmitted as a set of messages $\{\mathbf{M}_{j}^k\}_{j \in \mathcal{N}}$ to the server. 
The server adds the dithers induced by the shared randomness to the received messages, collects the decoded messages and obtains the estimations $\hat{\mathbf{X}}_t^k$. 
Subsequently, the server computes the average of estimated model updates $\frac{1}{K}\sum_{k \in \mathcal{K}}\hat{\mathbf{X}}_{t}^k$, which ensures the DP guarantee against the honest but inquisitive clients.
The performance of CEPAM will be analyzed in Section~\ref{sec:perf_analysis}. 


\begin{algorithm}[ht]
\caption{CEPAM} \label{alg:UnivPriv}
\begin{algorithmic}[1]
\State \textbf{Inputs:} 
Number of total iterations $T$, number of local iterations $\tau$, number of clients $K$, local datasets $\{\mathcal{D}^{(k)}\}_{k \in \mathcal{K}}$, loss function $\ell(\cdot,\cdot)$, threshold $\gamma>0$ 
\State \textbf{Output:} Global optimized model $\mathbf{W}_T$
\State \textbf{Initialization:} Client $k$ and the server agree on privacy budget $\epsilon>0$ and privacy relaxation $\delta$ for $(\epsilon,\delta)$-DP (or privacy budget $\epsilon>0$ for $\epsilon$-DP), shared seed $s_k$, lattice dimension $n$, generator matrix $\mathbf{G}$, 
privacy-preserving noise $\mathbf{Z}\sim f$ with noise variance $\mathrm{Var}(f)>0$, latent variable $U \sim g(u):=\mu(L_{u}^{+}(f))$, initial model parameter vector $\mathbf{W}_{0} \in \mathbb{R}^m$ 
\State \textbf{Protocol at client $k$:}
\For{$t+1 \notin \mathcal{T}_T$}
\State Receive $\mathbf{W}_{t}$ from the server or use $\mathbf{W}_0$ if $t=0$
\State Set $\mathbf{W}_{t}^k  \leftarrow \mathbf{W}_{t}$
\For{$t'=1$ \textbf{to} $\tau$}
\State Compute $\mathbf{W}_{t+t'}^{k} \leftarrow \mathbf{W}_{t+t'-1}^{k} - \eta_{t+t'-1} \nabla F_{k}^{j_{t+t'-1}^{k}}(\mathbf{W}_{t+t'-1}^{k})$
\EndFor
\State Compute $\mathbf{X}_{t+\tau}^{k} \leftarrow \mathbf{W}_{t+\tau}^{k}-\mathbf{W}_t$
\State Compute $\tilde{\mathbf{X}}_{t+\tau}^{k}\leftarrow\mathbf{X}_{t+\tau}^{k}/\max\{1,\Vert\mathbf{X}_{t+\tau}^{k}\Vert_2/\gamma\}$ 
\Comment{Perform norm clipping}
\State Run subroutine $\mathrm{ENCODE}(\tilde{\mathbf{X}}_{t+\tau}^{k}, s_k, (f, g), N)$
\Comment{See Algorithm~\ref{alg:UnivPrivEnc}}
\State Send $\{(H_{t+\tau,j}^k,\mathbf{M}_{t+\tau,j}^k)\}_{j \in \mathcal{N}}$ to server, using $N\cdot\left(H(\mathrm{Geom}(p(U)\;|\;U)+H(\lceil \log |\mathcal{M}(U)|\rceil\;|\;U)\right)$ bits  
\EndFor 
\State \textbf{Protocol at the server:}
\For{$t+\tau \in \mathcal{T}_T$}
\State Receive $\{H_{t+\tau,j}^k, \mathbf{M}_{t+\tau,j}^k\}_{j \in \mathcal{N}}$ from clients
\For{$k \in \mathcal{K}$}
\State Run subroutine $\mathrm{DECODE}(\{(H_{t+\tau,j}^k,\mathbf{M}_{t+\tau,j}^k)\}_{j \in \mathcal{N}}, s_k, g, N)$ 
\Comment{See Algorithm~\ref{alg:UnivPrivDec}}
\EndFor
\State Compute  
$\hat{\mathbf{W}}_{t+\tau} \leftarrow \mathbf{W}_{t} +  \sum_{k \in \mathcal{K}} p_{k}\hat{\mathbf{X}}_{t+\tau}^{k}$
\State Set $\mathbf{W}_{t+\tau} \leftarrow \hat{\mathbf{W}}_{t+\tau}$ and broadcast $\mathbf{W}_{t+\tau}$ to all clients, or output $\mathbf{W}_T$ if $t+\tau=T$
\EndFor
\end{algorithmic}
\end{algorithm}


CEPAM consists of three stages: initialization, encoding at client, and decoding at server. 
The initialization stage involving both the clients and the server takes place before the start of the FL procedure, while the encoding and decoding stages are independently carried out by the clients and the server.
Each client is presumed to execute the same encoding function, ensuring consistency across all clients. 
Thus, we focus on the $k$-client, and the details of each stage are described as follows.

\emph{1) Initialization}: 
At the outset, client $k$ and the server agree on the privacy budget $\epsilon$ and privacy relaxation  $\delta$ for $(\epsilon,\delta)$ DP (or $\epsilon$ for $\epsilon$-DP) in accordance with the system requirements, along with the parameters for LRSUQ as specified in Section~\ref{sec:RSUQ}.
The latter involves sharing a small random seed $s_k$ between client $k$ and the server to serve as a source of common randomness while fixing the lattice dimension $n$ and a lattice generator matrix $\mathbf{G}$. 
Both client $k$ and the server use the same random seed to initialize their respective PRNGs $\mathfrak{P}$ and $\tilde{\mathfrak{P}}$, ensuring that the outputs of the two PRNGs remain identical.

\emph{2) Client}: 
At the end of each FL round of local training, when the model update $\mathbf{X}_{t+\tau}^{k} \in \mathbb{R}^m$ is ready for transmission through the binary uplink channel to the server,
client $k$ executes the encoding process. The update $\mathbf{X}_{t+\tau}^{k}$ is encoded into finite bit representations by a combination of LRSUQ and the entropy coding, which also jointly ensures the privacy guarantee with the decoding step at the server.
The encoding algorithm is summarized in Algorithm~\ref{alg:UnivPrivEnc}.

\begin{algorithm}[ht]
\caption{$\mathrm{ENCODE}(\tilde{\mathbf{X}}_{t+\tau}^{k},
s_k, (f, g), N)$} \label{alg:UnivPrivEnc}
\begin{algorithmic}[1]
\State \textbf{Inputs:} 
Clipped 
model update vector $\tilde{\mathbf{X}}_{t+\tau}^{k}$,
random seed $s_k$, pair of pdfs $(f,g)$, number of sub-vectors $N$
\State \textbf{Output:} Set of messages $\{(H_{t+\tau,j}^k,\mathbf{M}_{t+\tau,j}^k)\}_{j \in \mathcal{N}}$
\State Partition $\tilde{\mathbf{X}}_{t+\tau}^{k}$ to $\{\tilde{\mathbf{X}}_{t+\tau,j}^{k} \}_{j \in \mathcal{N}}$
\For{$j=1$ \textbf{to} $N$}
\State Sample $U_{t+\tau,j}^k \sim g$ by  $\mathfrak{P}^k$ \Comment{Initiated by seed $s_k$}
\For{$i=1, 2, \ldots$} 
\Comment{Perform RSUQ}
\State \label{step:RS1} Sample 
$\mathbf{V}_{t+\tau,j,i}^k \sim\mathrm{Unif}(\mathcal{P})$ by $\mathfrak{P}^k$
\Comment{Initiated by seed $s_k$}
\State Find unique $\mathbf{M}_{t+\tau,j}^k \leftarrow Q_{\mathcal{P}}(\tilde{\mathbf{X}}_{t+\tau,j}^{k}/\beta(U_{t+\tau,j}^k)-\mathbf{V}_{t+\tau,j,i}^k)\in\mathbf{G}\mathbb{Z}^{n}
$ \Comment{Perform encoding}
\State \label{step:RS2} Check if $\beta(U_{t+\tau,j}^k)\cdot(\mathbf{M}_{t+\tau,j}^k+\mathbf{V}_{t+\tau,j,i}^k)-\tilde{\mathbf{X}}_{t+\tau,j}^{k}\in L_{U_{t+\tau,j}^k}^{+}(f)$ 
\If{Yes} \Comment{Perform rejection sampling}
\State $H_{t+\tau,j}^k \leftarrow i$;
 Return $(H_{t+\tau,j}^k,\mathbf{M}_{t+\tau,j}^k)$
\Else
\State Reject $i$ and repeat Step~\ref{step:RS1}-\ref{step:RS2} with $i+1$
\EndIf
\EndFor
\EndFor
\State \textbf{return} $\{(H_{t+\tau,j}^k,\mathbf{M}_{t+\tau,j}^k)\}_{j \in \mathcal{N}}$ 
\end{algorithmic}    
\end{algorithm}


\textbf{Quantization:} 
The $k$-th client clips the update $\mathbf{X}_{t+\tau}^{k}$ by a real number $\gamma > 0$ to $\tilde{\mathbf{X}}_{t+\tau}^{k}$ (so that the $2$-norm of $\tilde{\mathbf{X}}_{t+\tau}^{k}$ is bounded by $\gamma$), and then divides $\tilde{\mathbf{X}}_{t+\tau}^{k}$ into $N:=\left\lceil \frac{m}{n} \right\rceil$ distinct $n$-dimensional sub-vectors $\tilde{\mathbf{X}}_{t+\tau,j}^{k} \in \mathbb{R}^n$, $j \in \{1, \ldots, N\}=:
\mathcal{N}$.

Instead of using the universal quantization, we replace it by LRSUQ
(see Definition~\ref{def:rej_samp_quant_layer}).
Let $\mathbf{Z}\sim f$ denote the the intended privacy-preserving noise random vector with  mean $\boldsymbol{\mu}:=\mathbb{E}[\mathbf{Z}]=\mathbf{0}$ and 
noise variance $\mathrm{Var}(f):=\mathrm{Var}(\mathbf{Z})$, and let $g(u):=\mu(L_{u}^{+}(f))$ be the pdf of some latent variable.
To carry out LRSUQ at client $k$, the encoder 
at the client observes for $j \in \mathcal{N}$, $\tilde{\mathbf{X}}_{t+\tau,j}^{k}\in\mathbb{R}^{n}$, 
generates $\mathbf{V}_{t+\tau, j,1}^k,\mathbf{V}_{t+\tau, j,2}^k,\ldots\stackrel{iid}{\sim}\mathrm{Unif}(\mathcal{P})$ and $U_{t+\tau,j}^k \sim g$ by the PRNG $\mathfrak{P}^k$ until iteration $H_{t+\tau,j}^k$ satisfying $\beta(U_{t+\tau,j}^k)\cdot\big(Q_{\mathcal{P}}(\tilde{\mathbf{X}}_{t+\tau,j}^{k}/\beta(U_{t+\tau,j}^k)-\mathbf{V}_{H_{t+\tau,j}^k})+\mathbf{V}_{H_{t+\tau,j}^k}\big)-\tilde{\mathbf{X}}_{t+\tau,j}^{k}\in L_{U_{t+\tau,j}^k}^{+}(f)$,
computes $\mathbf{M}_{t+\tau,j}^k:=Q_{\mathcal{P}}(\tilde{\mathbf{X}}_{t+\tau,j}^{k}/\beta(U_{t+\tau,j}^k)-\mathbf{V}_{H_{t+\tau,j}^k})\in\mathbf{G}\mathbb{Z}^{n}$,
and encodes and transmits $(H_{t+\tau,j}^k,\mathbf{M}_{t+\tau,j}^k)$.\footnote{By setting $U_{t+\tau,j}^k = \emptyset$, we could retrieve the simpler RSUQ as defined in Definition~\ref{def:rej_samp_quant}.}  

\emph{3) Server}: 
The server observes
$\{(H_{t+\tau,j}^k,\mathbf{M}_{t+\tau,j}^k)\}_{j \in \mathcal{N}}$ from client $k$.
Using the shared randomness between client $k$ and the server, i.e., the shared random seed $s_k$,  the same realizations of $\mathbf{V}_{H_{t+\tau,j}^k}$’s and $U_{t+\tau,j}^k$’s generated by client $k$ can also be obtained by the server.
More specifically, the decoder 
generates  $U_{t+\tau,j}^k \sim g$ and  $\mathbf{V}_{t+\tau,j,1},\ldots,\mathbf{V}_{t+\tau,j,H_{t+\tau,j}^k}\stackrel{iid}{\sim}\mathrm{Unif}(\mathcal{P})$ by the PRNG $\tilde{\mathfrak{P}}^k$, and outputs $\mathbf{Y}_{t+\tau,j}^k=\beta(U_{t+\tau,j}^k)(\mathbf{M}_{t+\tau,j}^k+\mathbf{V}_{H_{t+\tau,j}^k})$.
Subsequently, the decoder collects the sub-vectors $\{\mathbf{Y}_{t+\tau,j}^k\}_{j \in \mathcal{N}}$ into an 
estimated vector $\hat{\mathbf{X}}_{t+\tau}^{k}$. 
of the update $\mathbf{X}_{t+\tau}^{k}$. 
The decoding algorithm is summarized in Algorithm~\ref{alg:UnivPrivDec}.

\begin{algorithm}[t]
\caption{$\mathrm{DECODE}(\{(H_{t+\tau,j}^k,\mathbf{M}_{t+\tau,j}^k)\}_{j \in \mathcal{N}}, s_k, g, N)$} \label{alg:UnivPrivDec}
\begin{algorithmic}[1]
\State \textbf{Inputs:} Set of messages $\{(H_{t+\tau,j}^k,\mathbf{M}_{t+\tau,j}^k)\}_{j \in \mathcal{N}}$, 
random seed $s_k$, pdf $g$,  number of sub-vectors $N$
\State \textbf{Output:} Estimated model update $\hat{\mathbf{X}}_{t+\tau}^{k}$
\State Use $\tilde{\mathfrak{P}}^k$ to sample $U_{t+\tau,j}^k \sim g$  \Comment{Initiated by seed $s_k$}
\For{$j \in \mathcal{N}$}
\State Use $\tilde{\mathfrak{P}}^k$ to sample $\mathbf{V}_{t+\tau,j,1},\ldots,\mathbf{V}_{t+\tau,j,H_{t+\tau,j}^k}\stackrel{iid}{\sim}\mathrm{Unif}(\mathcal{P})$ \Comment{Initiated by seed $s_k$}
\State \label{step:PP} Compute  $\mathbf{Y}_{t+\tau,j}^k\leftarrow\beta(U_{t+\tau,j}^k)(\mathbf{M}_{t+\tau,j}^k+\mathbf{V}_{H_{t+\tau,j}^k})$ \Comment{Perform decoding, inducing privacy protection}
\EndFor
\State Collect $\{\mathbf{Y}_{t+\tau,j}^k\}_{j \in \mathcal{N}}$ into  $\mathbf{Y}_{t+\tau}^k$ and set  $\hat{\mathbf{X}}_{t+\tau}^{k}\leftarrow 
\mathbf{Y}_{t+\tau}^k$ 
\State \textbf{return} $\hat{\mathbf{X}}_{t+\tau}^{k}$
\end{algorithmic}
\end{algorithm}


\textbf{Privacy enhancement}: 
The privacy is ensured by LRSUQ executed cooperatively between the clients and the server in compliance with the threat model and the defined privacy guarantee.  
Through the decoding process, it is ensured that $\{\mathbf{Y}_{t+\tau,j}^k\}_{j \in \mathcal{N}}$ comprises noisy estimates of $\{\tilde{\mathbf{X}}_{t+\tau,j}^{k}\}_{j \in \mathcal{N}}$, i.e., the estimated model updates at the server are noisy estimates of the clipped model updates at the clients, thus establishing a privacy mechanism.
Refer to Step~\ref{step:PP} in Algorithm~\ref{alg:UnivPrivDec}.
Specifically, 
in Section~\ref{sec:GaussMec}, we prove that when using the Gaussian mechanism, the global average of estimated model updates satisfies $(\epsilon,\delta)$-DP requirement by specifying appropriate $f$ and $g$, while in Section~\ref{sec:LapMech}, when using the Laplace mechanism, the estimated model updates satisfies $\epsilon$-DP requirement by specifying appropriate $f$ and $g$.

\section{Performance Analysis} \label{sec:perf_analysis}

In this section, we study the performance of CEPAM. 
This includes characterizing its privacy guarantees and compression capabilities, 
followed by an exploration of its distortion bounds and convergence analysis. 
We begin by stating a useful lemma for the subsequent analysis. 
The proof is given in Appendix~\ref{subsec:pf_LRSUQ_error}.
\smallskip
\begin{lem} \label{lem:LRSUQ_error}
The LRSUQ quantization errors $\{\tilde{\mathbf{Z}}_{t+\tau,j}^k:= \mathbf{Y}_{t+\tau,j}^k - \tilde{\mathbf{X}}_{t+\tau,j}^{k}\}_{j \in \mathcal{N}}$ are iid (over $k$ and $j$),  follows the pdf $f$, and independent of $\tilde{\mathbf{X}}_{t+\tau,j}^{k}$. 
\end{lem}

\subsection{Privacy} \label{sec:priv_anal} 

By using LRSUQ, we can construct privacy mechanisms that satisfy various privacy requirements by customizing the pair $(f,g)$.   
In the following subsections, we present examples of $(f, g)$ pairs to illustrate the construction of Gaussian and Laplace mechanisms, along with demonstrating the associated privacy guarantees.
For simplicity, we assume that $p_k = 1/K$. 

\subsubsection{Gaussian mechanism} \label{sec:GaussMec}

If $g$ follows a chi-squared distribution with $n+2$ degrees of freedom, i.e., $g \sim \chi_{n+2}^2$, 
then $f$ follows a Gaussian distribution.
This guarantees that the global average of estimated model updates satisfies $(\epsilon, \delta)$-DP. 
The proof is given in Appendix~\ref{subsec:pf_Gaussian_Priv}.

\smallskip
\begin{thm} \label{thm:Gaussian_Priv}
Set  $\tilde{U}_{t+\tau,j}^k \sim g = \chi_{n+2}^2$ and $\tilde{\mathbf{Z}}_{t+\tau,j}^k|\{\tilde{U}_{t+\tau,j}^k=u\} \sim \mathrm{Unif}(\sigma \sqrt{u} B^{n})$ for every $k, j$, 
the resulting mechanism, i.e., Step~\ref{step:PP} in Algorithm~\ref{alg:UnivPrivDec},
is Gaussian.
The average $\frac{1}{K} \sum_{k \in \mathcal{K}} \hat{\mathbf{X}}_{t+\tau}^{k}$ is a noisy estimate of the average of the clipped
model updates such that    
\begin{equation} \label{eq:Gaussian_noisy_eq}
    \frac{1}{K} \sum_{k \in \mathcal{K}} \hat{\mathbf{X}}_{t+\tau}^{k}= \frac{1}{K}\sum_{k \in \mathcal{K}}\tilde{\mathbf{X}}_{t+\tau}^{k}+ \mathcal{N}\bigg(\mathbf{0},\frac{\sigma^2}{K}\mathbf{I}_{m}\bigg),
\end{equation}
where $\mathbf{I}_m$ is the $m \times m$ identity matrix. 
Hence, for every $\tilde{\epsilon} >0$, $\frac{1}{K} \sum_{k \in \mathcal{K}} \hat{\mathbf{X}}_{t+\tau}^{k}$ satisfies $(\epsilon,\delta)$-DP in one round against clients for  $\epsilon = \log \left(1 + p(e^{\tilde{\epsilon}} - 1)\right)$ where $p = 1- \left(1-\frac{1}{|\mathcal{D}^{(k)}|}\right)^{\tau}$ and 
\begin{equation} \label{eq:PrivAmpGSRFL}
    \delta = \sum_{j=1}^{\tau} \binom{\tau}{j}\bigg(\frac{1}{|\mathcal{D}^{(k)}|}\bigg)^j\bigg(1-\frac{1}{|\mathcal{D}^{(k)}|}\bigg)^{\tau-j}\frac{(e^{\tilde{\epsilon}}-1)}{e^{\tilde{\epsilon}/j}-1}\left(\Phi\bigg(\frac{\tau\gamma}{\sqrt{K}\sigma} - \frac{\sqrt{K}\tilde{\epsilon} \sigma}{2j\tau\gamma}\bigg) -  e^{\tilde{\epsilon}/j} \Phi\bigg(-\frac{\tau\gamma}{\sqrt{K}\sigma} - \frac{\sqrt{K}\tilde{\epsilon} \sigma}{2j\tau\gamma}\bigg)\right).
\end{equation}
\end{thm}

From  Lemma~\ref{lem:subsampledGauss}, the privacy profile $\delta$ of the subsampled Gaussian mechanism (Equation~\eqref{eq:Gaussian_noisy_eq}) depends on the size $n$ of the dataset due to privacy amplification via subsampling with replacement, and on the global $\ell_2$-sensitivity $\Delta_2$. 
Note that the size of the dataset is $n=|\mathcal{D}^{(k)}|$ and the global $\ell_2$-sensitivity is at most $\frac{2\tau\gamma}{K}$.

\subsubsection{Laplace mechanism} \label{sec:LapMech}

If $g$ follows a Gamma distribution $\mathrm{Gamma}(2,1)$, 
    then $f$ follows a Laplace distribution.\footnote{ Our CEPAM can be extended to include multivariate Laplace distributions \cite{andres2013geo} or $t$-distributions \cite{Reimherr2019Elliptical}, which are  left for future work.}
 This guarantees that estimated model updates satisfy $\epsilon$-DP. 
 The proof is given in 
 Appendix~\ref{subsec:pf_Laplace_Priv}.
\begin{thm} \label{thm:Laplace_privy}
Set  $\tilde{U}_{t+\tau,j}^k \sim g = \mathrm{Gamma}(2,1)$  and $\tilde{Z}_{t+\tau,j}^k|\{\tilde{U}_{t+\tau,j}^k=u\} \sim \mathrm{Unif}((-bu , bu))$ for every $k$ and $j$, the resulting mechanism is Laplace.
The estimator $ \hat{\mathbf{X}}_{t+\tau}^{k}$ is a noisy estimate of the clipped model update $\tilde{\mathbf{X}}_{t+\tau}^{k}$ such that    
\begin{equation} \label{eq:Laplace_noisy_eq}
    \hat{\mathbf{X}}_{t+\tau}^{k}= \tilde{\mathbf{X}}_{t+\tau}^{k}+ \mathrm{Lap}\left(\mathbf{0},b\mathbf{I}_{m}\right),
\end{equation}
where $\mathbf{I}_m$ is the $m \times m$ identity matrix. 
Hence, for every $ \hat{\mathbf{X}}_{t+\tau}^{k}$ satisfies $(\epsilon,0)$-DP in one round against clients for  $\epsilon = \log \left(1 + p(e^{\tilde{\epsilon}} - 1)\right)$ where $p = 1- \left(1-\frac{1}{|\mathcal{D}^{(k)}|}\right)^{\tau}$, and $\delta=0$ 
provided that $\tilde{\epsilon} \ge 2\tau\gamma/b$.
\end{thm}

\subsection{Compression} 
 
Every client in CEPAM is required to transmit the set of message pairs $(H_{t+\tau,j}^k,\mathbf{M}_{t+\tau,j}^k)_{j\in \mathcal{N}}$ per communication round.
If we know that $\tilde{\mathbf{X}}_{t+\tau,j}^{k}\in \mathcal{X}$, 
then we can compress $H_{t+\tau,j}^k$ using the optimal prefix-free code \cite{Golomb1966enc,Gallager1975Geom} for  $\mathrm{Geom}(p(u))|\{U_{t+\tau,j}^k=u\}$, where $p(u):=\mu(L_{u}^{+}(f))/\mu(\beta(u)\mathcal{P})$, and compress $\mathbf{M}_{t+\tau,j}^k |\{U_{t+\tau,j}^k=u\} \in \mathcal{M} := (\mathcal{X}+L_{u}^{+}(f)-\beta(u)\mathcal{P}) \cap \mathbf{G}\mathbb{Z}^{n}
$ using $H(\lceil \log |\mathcal{M}(U)|\rceil\;|\;U)$ bits. 
Therefore, the total communication cost per communication round per client is at most $ N\cdot \left(H(\mathrm{Geom}(p(U)\;|\;U)+H(\lceil \log |\mathcal{M}(U)|\rceil\;|\;U)\right)$ bits.

\section{Numerical Evaluations} \label{sec:eval}

In this section, we evaluate the performance of  CEPAM.\footnote{The source code used in our numerical evaluations is available at \url{https://github.com/yokiwuuu/CEPAM.git}}  
We begin by detailing our experimental setup, which includes the types of datasets, model architectures, and training configurations in Section~\ref{sec:exp_setup}. 
Afterwards, we present comprehensive experimental results to demonstrate the effectiveness of CEPAM by comparing it to several baselines realizing the privacy protection and quantization in the FL framework 
in Section~\ref{sec:FLconv}.

We also compare the accuracy-privacy trade-off between CEPAM and the Gaussian-mechanism-then-quantize approach in Section~\ref{sec:PrivAccur}.
All experiments were executed on a server equipped with dual Intel Xeon Gold 6326 CPUs (48 cores and 96 threads in total), 256 GiB RAM, and two NVIDIA RTX A6000 GPUs (each with 48GB VRAM), running Ubuntu 22.04.5 LTS. The implementation was based on Python 3.13.5 and PyTorch 2.7.1+cu126, with CUDA 11.5 and NVIDIA driver version 565.57.01.

\subsection{Experimental Setup} \label{sec:exp_setup}

\subsubsection{Datasets}
We evaluate CEPAM on the standard image classification benchmark MNIST, which consists of $28 \times 28$ grayscale handwritten digits images divided into 60,000 training examples and 10,000 test examples.
The training examples and the test examples are equally distributed among $K=30$ clients. 
For simplicity, we set $p_k=1/K$ for $k \in \mathcal{K}$.

\subsubsection{Learning Architecture}
We evaluate CEPAM using two different neural network structures: a multi-layer perceptron (MLP) with two hidden layers and intermediate ReLU activations; 
and a convolutional neural network (CNN) composed of two convolutional layers followed by two fully-connected
ones, with intermediate ReLU activations and max-pooling layers. 
All two models use a softmax output layer. 
There are 6422 learnable parameters for CNN and 25818 learnable parameters for MLP.

\subsubsection{Baselines}
We compare CEPAM to other baselines:
\begin{itemize}
    \item FL: vanilla FL without any privacy or compression.
    \item FL+SDQ : FL with scalar SDQ-based compression, no privacy.
    \item FL+\{Gaussian, Laplace\}: FL with one-dimensional Gaussian or Laplace mechanism, no compression
    \item FL+\{Gaussian, Laplace\}+SDQ: FL with approach that applies one-dimensional Gaussian or Laplace mechanism followed by scalar SDQ.
    \item CEPAM-\{Gaussian, Laplace\}: CEPAM achieves privacy and quantization jointly through LRSUQ to construct Gaussian or Laplace mechanisms, namely CEPAM-Gaussian or CEPAM-Laplace. In particular, for CEPAM-Gaussian, we evaluate three different cases of LRSUQ that simulate Gaussian noise for dimension $n=1,2,3$, 
    where we use a scaled integer lattice $\alpha \mathbb{Z}^n$, where $\alpha = 10^{-5}$.\footnote{To simplify implementation, we choose to use the integer lattice due to its well-established decoding algorithm \cite{conway1982fast}. The time complexity of the rejection sampling step can be improved by using a lattice with higher packing density.} The corresponding basic cells $\alpha (-0.5,0.5]$, $\alpha (-0.5,0.5]^2$, and $\alpha (-0.5,0.5]^3$. 
    For CEPAM-Laplace using LRSUQ to simulate Laplace noise, we use the scaled integer lattice $\alpha\mathbb{Z}$ with the basic cell $(-0.5,0.5]$ for the same $\alpha = 10^{-5}$.
\end{itemize}

\subsubsection{Training Configurations}

We select the momentum SGD as the optimizer, where the momentum is set to $0.9$.  
The local iterations $\tau$ per communication round is set to 15.
While we set the initial learning rate to be $0.01$, we also implement an adaptive learning rate scheme that reduces the learning rate by a factor of $0.5$ when the validation accuracy plateaus for $10$ consecutive rounds.

\subsubsection{Repetition Strategy} For each set of parameters and baseline method, the training process was repeated 10 times with varying random seeds, and the average performance across these runs was reported.

\subsection{FL Convergence} \label{sec:FLconv}

We report the FL convergence in terms of accuracy for  CEPAM-Gaussian using MLP and CNN architectures and CEPAM-Laplace using CNN architecture on MNIST dataset.

\subsubsection{Gaussian Mechanism} \label{sec:FLconvGauss}

We set $\sigma=0.001$ and $\tilde{\epsilon}=5.9$ for the base Gaussian mechanisms in CEPAM-Gaussian for $n=1,2,3$ with variances $0.001^2$, $2\times0.001^2$, $3\times0.001^2$, respectively. 
By Theorem~\ref{thm:Gaussian_Priv}, all the clients' composite Gaussian mechanisms achieve $(\epsilon=1.45, \delta=2.48\times 10^{-2})$-DP. 
We also set $\sigma=0.001$ for both FL+Gaussian and FL+Gaussian+SDQ.

Figure~\ref{fig:CEPAM-Gaussian} shows the validation accuracy of CEPAM-Gaussian over communication rounds using MLP and CNN model, respectively. 
Note that for MLP, the cases of CEPAM-Gaussian ($n=1,2,3$) consistently outperform and achieve higher accuracy than all the baselines. 
Similar trend can be observed for CNN.

\begin{figure}[ht]
    \centering
    \includegraphics[width=\linewidth]{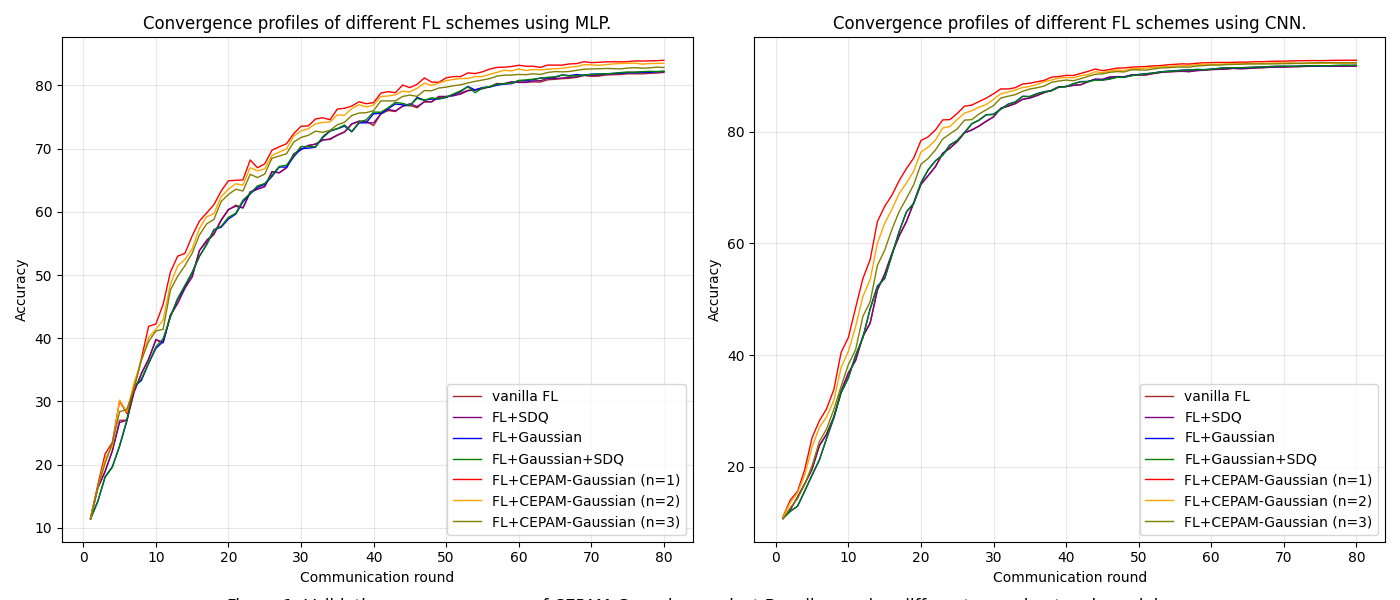}
    \caption{Convergence profile of different FL schemes for CEPAM-Gaussian.}
    \label{fig:CEPAM-Gaussian}
\end{figure}

In Table~\ref{table:AccGaussian}, we report the test accuracy of CEPAM-Gaussian for different  architectures (MLP and CNN) using MNIST with their 95\% confidence intervals. CEPAM-Gaussian demonstrates better performance, achieving an improvement of 0.6-2.0\% and 0.4-1.0\% in accuracy compared to other baselines for MLP and CNN, respectively, suggesting that CEPAM can be beneficial for various learning models.
Furthermore, all cases of CEPAM-Gaussian have approximately the same accuracy. 
This is because the added noise per dimension is consistent, i.e., $\mathrm{Var}(f)/n=\sigma^2$.
However, higher-dimensional LRSUQ has a better compression ratio than scalar counterparts   \cite{hasircioglu2023communication,yan2023layered} due to its nature as a vector quantizer.
In contrast, for baselines such as the Gaussian-mechanism-then-quantize approach (Gaussian+SDQ), the total noise comprises privacy noise and quantization error, 
leading to more distortion and consequently lower accuracy compared to LRSUQ.
It is worth noting that adding a minor level of distortion as a regularizer during training deep models can potentially enhance the performance of the converged model \cite{An1996AddingNoise}.
This finding is corroborated in our experimental results, where CEPAM-Gaussian achieves slightly better accuracy than vanilla FL for both architectures.


\begin{table}[ht]
\centering
\renewcommand{\arraystretch}{0.8} 
\begin{tabular}{ccc}
\toprule
Baselines   &  MLP (\%)  & CNN (\%) \\
\midrule
FL          &  $82.06 \pm 1.35$ & $91.70 \pm 0.66$ \\
FL+SDQ      &  $82.17 \pm 1.31$ & $91.73 \pm 0.68$ \\
FL+Gaussian &  $82.23 \pm 1.23$ & $91.84 \pm 0.72$ \\
FL+Gaussian+SDQ  & $82.26 \pm 1.20$ & $91.84 \pm 0.70$ \\
CEPAM-Gaussian ($n=1$) & $83.99 \pm 0.92$ & $92.74 \pm 0.70$ \\
CEPAM-Gaussian ($n=2$) & $83.52 \pm 1.00$ & $92.26 \pm 0.64$ \\
CEPAM-Gaussian ($n=3$) & $82.84 \pm 1.19$ & $92.29 \pm 0.72$ \\
\bottomrule
\end{tabular}
\vspace{-2mm}
\caption{Test Accuracy for MNIST}
\label{table:AccGaussian}
\end{table}

\subsubsection{Laplace Mechanism} \label{sec:FLconvLap}
 We set $b=0.001$ and $\tilde{\epsilon}=$30000 for the base Laplace mechanism in CEPAM-Laplace with a variance of $2 \times 0.001^2$, and the composite Laplace mechanism achieves ($\epsilon = 29995$)-DP. We also set $b=0.001$ for both FL+Laplace and FL+Laplace+SDQ. 
Note that a higher privacy budget is required to ensure $\epsilon$-DP for CEPAM-Laplace under privacy amplification with subsampling analysis.
However, we might achieve a lower privacy budget by relaxing $\delta$. 
The trade-off is left for the future study.

Figure~\ref{fig:CEPAM-Laplace} shows the validation accuracy of CEPAM-Laplace over communication rounds using the CNN models, respectively. 
Similar to CEPAM-Gaussian, we observe that CEPAM-Laplace consistently outperforms other baselines. 

\begin{figure}
    \centering
    \includegraphics[width=0.6\linewidth]{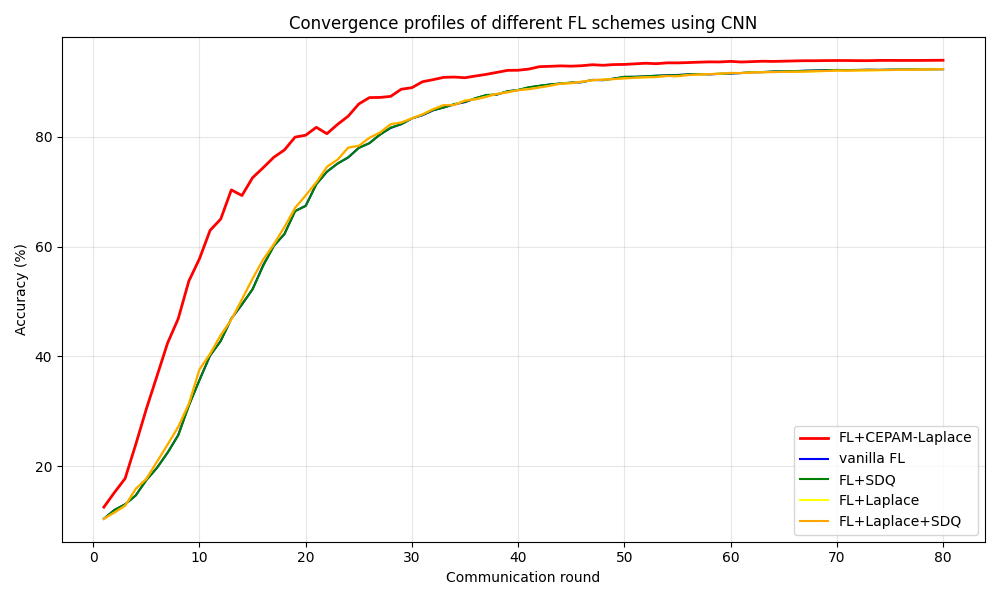}
    \caption{Convergence profiles of different FL schemes for CEPAM-Laplace.}
    \label{fig:CEPAM-Laplace}
\end{figure}

In Table~\ref{table:AccLaplace}, we also report the performance of test accuracy of CEPAM-Laplace. 
We observe that CEPAM-Laplace achieves 
an improvement of about 1.5-1.6\% in accuracy compared to other baselines for CNN, respectively, 
which suggests that CEPAM-Laplace can be beneficial for many learning models.

\begin{table}[ht]
\centering
\caption{Test Accuracy (Laplace) for MNIST}
\label{table:AccLaplace}
\begin{tabular}{cc}
\toprule
Baselines   & CNN (\%)\\
\midrule
FL          &  $92.92 \pm 0.61$\\
FL+SDQ      &  $92.93 \pm 0.62$ \\
FL+Laplace  &  $92.95 \pm 0.58$ \\
FL+Laplace+SDQ  & $92.99 \pm 0.57$ \\
CEPAM-Laplace   & $94.53 \pm 0.38$ \\
\bottomrule
\end{tabular}
\end{table}

\subsection{Privacy-Accuracy Trade-off} \label{sec:PrivAccur}

We evaluate the privacy-accuracy trade-off using two different values of privacy relaxation $\delta$. 
In the following, the parameters are computed according to Theorem~\ref{thm:Gaussian_Priv}.
For $\delta = 0.01$, we conducted the experiments by varying privacy budget $\epsilon$ from 1 to 10, the corresponding $\sigma$ are $\{0.1375, 0.0864, 0.0602, 0.0431, 0.0307, 0.0231, 0.0141, 0.0083, 0.00364, 0.000159\}$ respectively. Similarly, for $\delta = 0.015$, we conducted the experiments by varying privacy budget $\epsilon$ from 1 to 5, the corresponding $\sigma$ are $\{0.0857, 0.0445, 0.0236, 0.0101, 0.0005\}$ respectively. Again, for each set of parameters, the experiments are simulated 10 times, and the results are averaged.

Figure~\ref{fig:AccPriv} illustrates the trade-off between learning performance, measured in terms of test accuracy, and the privacy budget between CEPAM and the simple Gaussian-mechanism-then-quantize approach (Gaussian+SDQ) using a CNN architecture. 

Overall, CEPAM-Gaussian outperforms Gaussian+SDQ. 
The figure demonstrates that as more privacy budget is allocated, higher test accuracy can be achieved with CEPAM-Gaussian. 
However, for each $\delta$, there is a point of diminishing returns; once the privacy budget reaches a threshold. In particular, at $\epsilon \approx 7$ and $\epsilon \approx 4$ for $\delta = 0.01, 0.015$ respectively, the increase in test accuracy by further increasing the privacy budget becomes limited.
\begin{figure}[ht]
    \centering
    \includegraphics[width=0.6\textwidth]{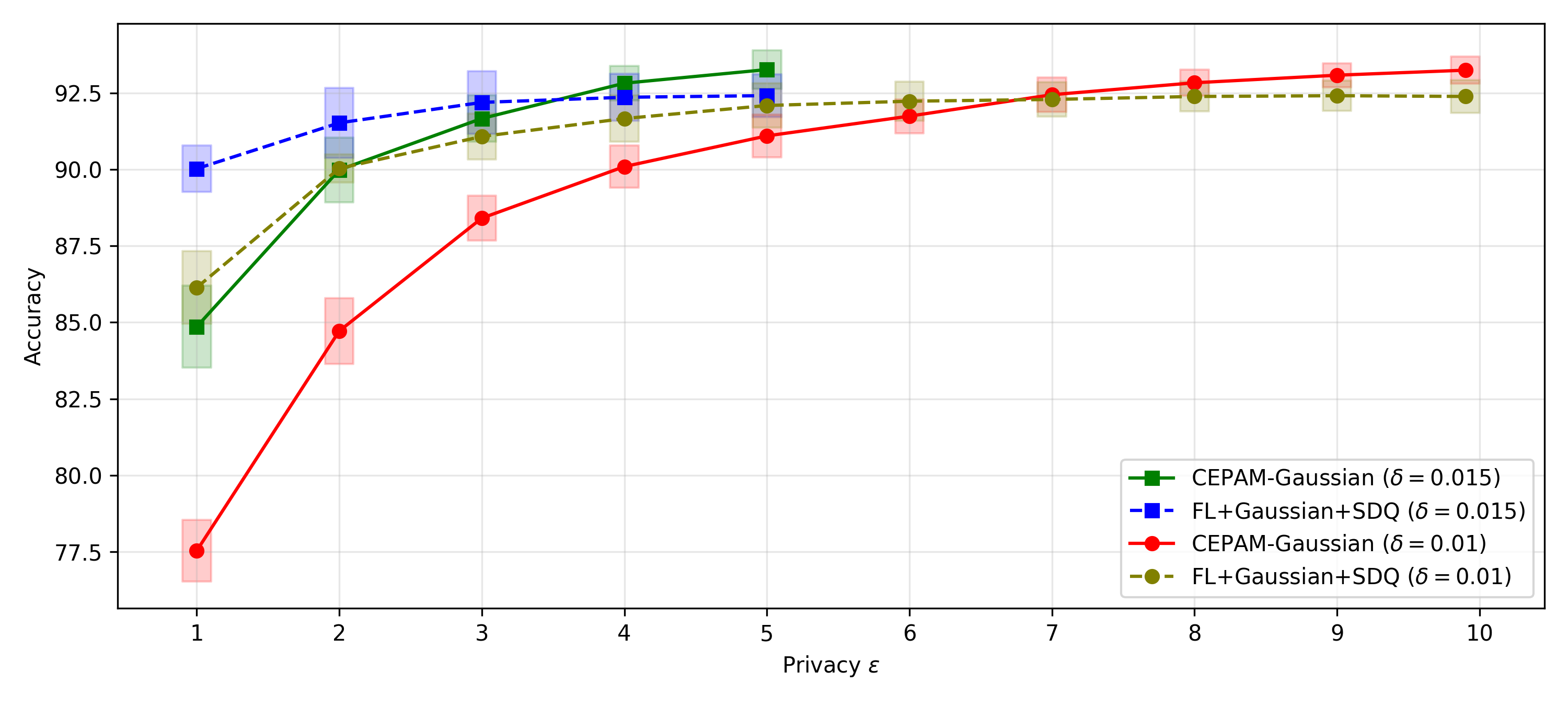} 
    \caption{Learning Accuracy and Privacy Trade-off using CNN}
    \label{fig:AccPriv}
\end{figure}

\section{Conclusion and Discussion}

In this paper, we introduced CEPAM, which can be used to achieved communication efficiency and privacy protection simultaneously in FL framework. 
We also proved the FL convergence bound for CEPAM and provided a privacy amplification analysis for CEPAM-Gaussian and CEPAM-Laplace. 
Furthermore, the experimental results confirmed our theoretical findings through a comparison of CEPAM with alternative methods using the MNIST dataset with MLP and CNN architectures. 
The experimental outcomes demonstrated that CEPAM-Gaussian and CEPAM-Laplace lead to a moderate enhancement in accuracy by 0.5-2.0\% and 1.5-2.7\%, respectively, when contrasted with other baseline approaches.

For potential future directions, it may be of interest to investigate the performance of CEPAM for the non-convex smooth objective function. 
Moreover, CEPAM could be expanded to incorporate other nonuniform continuous noise distributions, such as $t$-distributions or multivariate Laplace distributions, in order to provide alternative privacy mechanisms.

\bibliographystyle{IEEEtran}
\bibliography{ref}

\appendix

\section{Proofs for Section~\ref{sec:perf_analysis}}

\subsection{Proof of Lemma~\ref{lem:LRSUQ_error}\label{subsec:pf_LRSUQ_error}}
Since we are using LRSUQ for the encoding and  decoding of the $j$-th sub-vector $\tilde{\mathbf{X}}_{t+\tau,j}^{k}$ where $j \in \mathcal{N}$, Proposition~\ref{prop:LRSUQ_error} implies that, regardless of the statistical models of $\tilde{\mathbf{X}}_{t+\tau,j}^{k}$ where $j \in \mathcal{N}$, the quantization errors $\{\tilde{\mathbf{Z}}_{t+\tau,j}^k\}_{j \in \mathcal{N}}$ are iid (over $k$ and $j$) and follows the pdf $f$.

\subsection{Proof of Theorem~\ref{thm:Gaussian_Priv}\label{subsec:pf_Gaussian_Priv}}

Our proof relies on the following fact about the mixture of uniform distributions \cite{WalkerUniformP1999}. 

\begin{lem} \label{lem:mixtGaussian}
If $\mathbf{Z} \, |\, \{\tilde{U}=u\} \sim \mathrm{Unif}(\sigma \sqrt{u} B^{n})$ and $\tilde{U} \sim \chi^{2}_{n+2}$, then $\mathbf{Z} \sim \mathcal{N}(\mathbf{0},\sigma^2\mathbf{I}_{n})$.
\end{lem}

In other words, given the noise vector $\mathbf{Z}$, which follows a zero-mean multivariate Gaussian distribution with a covariance matrix $\sigma^2\mathbf{I}_{n}$, where $\mathbf{I}_n$ is the $n \times n$ identity matrix, i.e., $\mathbf{Z} \sim \mathcal{N}(\mathbf{0}, \sigma^2\mathbf{I}_{n})$, $\mathbf{Z}$ can be expressed as a mixture of uniform distributions over $n$-dimensional balls, with the latent variable $\tilde{U}$ following a chi-squared distribution with $n+2$ degrees of freedom.
By setting $\tilde{U}_{t+\tau,j}^k \sim g = \chi_{n+2}^2$ for every $k$ and $\tilde{\mathbf{Z}}_{t+\tau,j}^k\;|\;\{\tilde{U}_{t+\tau,j}^k=u\} \sim \mathrm{Unif}(\sigma \sqrt{u} B^{n})$ for $k$ and $j$, Lemma~\ref{lem:LRSUQ_error} and Lemma~\ref{lem:mixtGaussian} imply that the resulting mechanism, i.e., Step~\ref{step:PP} in Algorithm~\ref{alg:UnivPrivDec}, is Gaussian. 

We utilize the subsampling with replacement mechanism to model the $\tau$-iterations of local SGD at $k$-th client,\footnote{According to Equation~\eqref{eq:local_sgd}, each $k$-th client performs $\tau$ iterations of local SGD with uniformly sampled indices from the dataset $\mathcal{D}^{(k)}$ (i.e.,  $p_j=1/|\mathcal{D}^{(k)}|=1/n$, where $j \in \{1,\ldots,n\}$ and $n=|\mathcal{D}^{(k)}|$ in~\eqref{eq:privAmpli_S} or~\eqref{eq:PrivAmplGSR}), this process can be regarded as subsampling with replacement, which outputs a multiset sample of size $\tau$ following a multinomial distribution $\mathrm{Mult}(\tau;p_1,\ldots,p_{n})$.} instead of employing Poisson subsampling for mini-batching \cite{Balle2018PrivAmpl} within a single local iteration as in \cite{hasircioglu2023communication}. 
Furthermore, the resulting subsampled Gaussian mechanism is applied to sub-vectors of dimension $n$ rather than to individual components.

For the $j$-th sub-vector in $\{\mathbf{Y}_{t+\tau,j}^k\}_{j \in \mathcal{N}}$ where $k \in \mathcal{K}$, we have $\tilde{\mathbf{Z}}_{t+\tau,j}^k \, |\, \{\tilde{U}_{t+\tau,j}^k=u\} \sim \mathrm{Unif}(\sigma \sqrt{u} B^{n})$. Since $\tilde{U}_{t+\tau,j}^k \sim \chi^{2}_{n+2}$, the noisy sub-vectors $\tilde{\mathbf{Z}}_{t+\tau,j}^k$ follows a Gaussian distribution $\mathcal{N}(\mathbf{0},\sigma^2\mathbf{I}_{n})$ by Lemma~\ref{lem:mixtGaussian}. 
Given that this is true for any client $k' \neq k$, the quantization noise of one client is Gaussian distributed from the perspective of any other client.
Considering the averaging operation at the central server, we obtain \eqref{eq:Gaussian_noisy_eq} by Lemma~\ref{lem:LRSUQ_error}, since sum of independent Gaussian random variables is a Gaussian random variable with appropriate parameters.

For the $(\epsilon, \delta)$-DP guarantee, we apply  Lemma~\ref{lem:subsampledGauss} on every $j$-th sub-vector $\hat{\mathbf{X}}_{t+\tau,j}^{k}$, which is the output of the subsampled Gaussian mechanism. 
Since the server averages sum of $\hat{\mathbf{X}}_{t+\tau,j}^{k}$'s over $K$ clients, the effect of a single data point $\xi \in \mathcal{D}^{(k)}$ is at most $\frac{\tau\gamma}{K}$.\footnote{Since there are $\tau$ iterations, any single data point $\xi \in \mathcal{D}^{(k)}$ can appear in the subsample at most $\tau$ times.}
Consequently, the $\ell_2$ sensitivity of $\frac{1}{K}\sum_{k \in \mathcal{K}}\mathbf{X}_{t+\tau,j}^{k}$ is at most $\frac{2\tau\gamma}{K}$.
Substituting $\Delta_{2} = \frac{2\tau\gamma}{K}$ in~\eqref{eq:PrivAmplGSR} and setting  $\tilde{\sigma}=\sigma/\sqrt{K}$, 
we obtain \eqref{eq:PrivAmpGSRFL}.


\subsection{Proof of Theorem~\ref{thm:Laplace_privy}\label{subsec:pf_Laplace_Priv}}

Our proof relies on the following fact about the mixture of uniform distributions.

\begin{lem} \label{lem:mixtLap}
If $Z \, |\, \{\tilde{U}=u\} \sim \mathrm{Unif}((-bu , bu))$ and $\tilde{U} \sim \mathrm{Gamma}(2,1)$, then $Z \sim \mathrm{Lap}(0,b)$.
\end{lem}

By setting $\tilde{U}_{t+\tau,j}^k \sim g = \mathrm{Gamma}(2,1)$ for every $k$ and $\tilde{Z}_{t+\tau,j}^k\;|\;\{\tilde{U}_{t+\tau,j}^k=u\} \sim \mathrm{Unif}((-bu , bu))$ for $k$ and $j$, Lemma~\ref{lem:LRSUQ_error} and Lemma~\ref{lem:mixtLap} imply that the resulting mechanism is Laplace.

The remaining proof is similar to that of Theorem~\ref{thm:Gaussian_Priv}, and the details are omitted. 
Note that the $\ell_1$-sensitivity $\Delta_1$ is at most $2\tau\gamma$.

\end{document}